\documentclass[11pt]{article}
\usepackage[margin=1in]{geometry} 
\usepackage[T1]{fontenc}
\usepackage[utf8]{inputenc}
\usepackage{amsmath}
\usepackage{amssymb}
\usepackage{amsfonts}
\usepackage{amsthm} 
\usepackage[english]{babel}
\usepackage{mathrsfs}
\usepackage{natbib}
\usepackage{enumitem} 
\usepackage{relsize}
\usepackage{multirow}
\usepackage{multicol}
\usepackage{bm}
\usepackage{mathtools}
\usepackage{bbm}
\usepackage{fancyhdr} 
\usepackage{titlesec} 
\usepackage{algorithm}
\usepackage{algorithmic}
\usepackage{hyperref} 
\usepackage{xcolor}

\date{}

\allowdisplaybreaks

\newcommand{\pii}{\boldsymbol{\pi}}
\newcommand{\pp}{\boldsymbol{P}}
\newcommand{\E} {\mathbb{E}}
\newcommand{\PP}{\mathbb{P}}
\newcommand{\lr}{\beta}
\newcommand{\KD} {\texttt{KD}}
\newcommand{\KI} {\texttt{KI}}
\newcommand{\eps}{\varepsilon}
\newcommand{\Dvs}{\Delta}
\newcommand{\Dvk}{\tilde{\Delta}}

\newcommand{\Dvsm}{M}
\newcommand{\Dvkm}{\tilde{M}}
\newcommand{\Dvm}{\bar{M}}
\renewcommand{\epsilon}{\eps}

\newtheorem{prop}{Proposition}
\newtheorem{cor}{Corollary}

\newtheorem{lemma}{Lemma}
\newtheorem{theorem}{Theorem}

\newcommand{\clip}{\mathrm{clip}}

\newcommand{\ceil}[1]{\left\lceil #1\right\rceil}

\newcommand\numberthis{\addtocounter{equation}{1}\tag{\theequation}}

\newcommand{\gap}{\mathrm{gap}}
\newcommand{\gapstar}{\mathrm{gap}^*}

\author{
  Sajad Khodadadian$^{1}$\thanks{Both authors contributed equally.}~,~
  Mehrdad Moharrami$^{2}$\footnotemark[1]\\[0.8em]
  $^{1}$Virginia Polytechnic Institute and State University,\\
  Grado Department of Industrial and Systems Engineering,\\
  \texttt{sajadk@vt.edu}\\[0.5em]
  $^{2}$University of Iowa, \\
  Department of Computer Science,\\
  \texttt{moharami@uiowa.edu}
}
\begin{document}

\title{Tail Distribution of Regret in Optimistic Reinforcement Learning}

\maketitle

\begin{abstract}
We derive instance-dependent tail bounds for the regret of optimism-based reinforcement learning in finite-horizon tabular Markov decision processes with unknown transition dynamics. We first study a UCBVI-type (model-based) algorithm and characterize the tail distribution of the cumulative regret $R_K$ over $K$ episodes via explicit bounds on $\PP(R_K \ge x)$, going beyond analyses limited to $\E[R_K]$ or a single high-probability quantile. We analyze two natural exploration-bonus schedules for UCBVI: (i) a $K$-dependent scheme that explicitly incorporates the total number of episodes $K$, and (ii) a $K$-independent (anytime) scheme that depends only on the current episode index. We then complement the model-based results with an analysis of optimistic Q-learning (model-free) under a $K$-dependent bonus schedule. 

Across both the model-based and model-free settings, we obtain upper bounds on $\PP(R_K \ge x)$ with a distinctive two-regime structure: a sub-Gaussian tail starting from an instance-dependent scale up to a transition threshold, followed by a sub-Weibull tail beyond that point. We further derive corresponding instance-dependent bounds on the expected regret $\E[R_K]$. The proposed algorithms depend on a tuning parameter $\alpha$, which balances the expected regret and the range over which the regret exhibits sub-Gaussian decay.

\end{abstract}

\section{Introduction}

We consider online Reinforcement Learning (RL) in finite-horizon Markov decision processes (MDPs) with unknown transition dynamics. In this setting, a learning agent interacts with an episodic environment for $K$ episodes of length $H$, and aims to minimize its cumulative regret with respect to an optimal policy. A long line of work has established sharp worst-case regret guarantees for model-based and model-free algorithms in tabular MDPs; see, for instance, the near-minimax optimality results for UCRL2 under the average-reward criterion \citep{jaksch2010near} and for UCBVI and its variants in the episodic setting \citep{azar2017minimax,zanette2019tighter}. These results show that, up to logarithmic factors, the optimal worst-case regret scales as $\widetilde{\mathcal O}(\sqrt{HSAK})$, where $S$ and $A$ are the sizes of the state and action spaces. 

More recently, there has been substantial progress on instance-dependent analysis in RL. Inspired by the classical gap-dependent theory of stochastic bandits, several works have derived regret bounds that depend on the suboptimality gaps of state--action pairs rather than only on $(S,A,H,K)$. For tabular MDPs, this includes non-asymptotic gap-dependent bounds for model-based algorithms \citep{simchowitz2019non}, fine-grained gap-dependent bounds via adaptive multi-step bootstrapping \citep{xu2021fine}, and gap-dependent guarantees for model-free $Q$-learning and its refinements \citep{yang2021q,velegkas2022reinforcement,zheng2025gap}. These results provide a more refined characterization of the regret in MDPs by incorporating an instance-dependent $Q^*$-gap, which measures the separation between optimal and suboptimal actions at each state.

At the same time, the tail behavior of regret in RL is still poorly understood. Most existing analyses deliver bounds of the form
\[
\PP\!\big(R_K \geq f(K,\delta)\big) \;\leq\; \delta,
\]
for a fixed confidence level $\delta \in (0,1)$, typically by combining a union bound with martingale concentration inequalities. While such high-probability guarantees are useful, they only describe the regret distribution at a single point and require the user to choose $\delta$ in advance. By contrast, in many safety-critical applications, such as healthcare decision making, robotic control in human environments, and financial risk management \cite{inamdar2024comprehensive}, one would like a more detailed characterization of the entire tail distribution of regret, that is, explicit bounds on $\PP(R_K \geq x)$ that hold for all thresholds $x \ge \E[R_K]$. 

In this paper, we take a step towards closing this gap by studying the concentration of regret for a model-based UCBVI-type algorithm and a model-free $Q$-learning algorithm in finite-horizon MDPs. We show that the tail of the cumulative regret, $\PP(R_K \ge x)$, admits an explicit, non-asymptotic sub-Gaussian or sub-Weibull bound, depending on the value of $x$ and the instance-dependent $Q^\star$-gap parameter. Our contributions are as follows:
\begin{itemize}[leftmargin=*]
    \item We study the tail behavior of the cumulative regret $R_K$ for a UCBVI-type algorithm in finite-horizon tabular MDPs under a global $Q^\star$-gap, and derive explicit, instance-dependent sub-Gaussian and sub-Weibull tail bounds for $\PP(R_K \ge x)$. Our analysis treats two families of exploration bonuses: a $K$-\emph{dependent} schedule whose bonuses depend on the total number of episodes $K$, and a $K$-\emph{independent} schedule whose bonuses depend only on the current episode index. 
    
    \item For the $K$-dependent scheme, we show that the regret exhibits a sub-Gaussian tail on an intermediate regime of $x$ and a sub-Weibull tail beyond $x>\tilde{\mathcal{O}}\!\big(H^{3/2}K^{(1+\alpha)/2}\big)$, where $\alpha \in [0,1]$ is a tuning parameter.

    \item For the $K$-independent scheme, we obtain a similar two-regime behavior, where the transition to sub-Weibull happens at $\tilde{\mathcal{O}}\!\big(H^{(3-\alpha)/(2-\alpha)}K^{1/(2-\alpha)}\big)$.

    \item We further show that the expected regret is $\mathbb{E}[R_K]        \;\le \;\tilde{\mathcal{O}}\!\Big(H^3SA(S+K^\alpha)/\gapstar\Big)$, where $\gapstar>0$ is the global $Q^\star$-gap. The tuning parameter $\alpha$ balances the expected regret and the range over which the regret exhibits a sub-Gaussian tail.

    \item Finally, we study the regret of a $Q$-learning-type algorithm in the $\KD$ setting, and we establish similar tail and expected regret bounds as above.
\end{itemize}

Our results provide the first characterization of the tail of the regret for standard optimistic algorithms in tabular episodic RL, and they suggest that the entire distribution of the regret, not just its expectation or a single high-probability quantile, can be controlled in an instance-dependent manner determined by the $Q^\star$-gap.

\section{Related Work}

\paragraph{Worst-case regret.}
The modern theory of online RL in finite MDPs starts from worst-case regret guarantees under the optimism in the face of uncertainty principle.  
In the average-reward setting, the UCRL2 algorithm of \citet{jaksch2010near} achieves near-minimax regret of order $\widetilde{\mathcal O}\!\big(DS\sqrt{AK}\big)$, where $D$ is the diameter of the MDP, and lower bound $\widetilde{\mathcal O}\!\big(\sqrt{DSAK}\big)$ shows that the dependence on $(D,S,A,K)$ is almost optimal.  
Subsequent refinements sharpened the constants and logarithmic factors, for instance by tightening the confidence sets around the transition kernel via empirical Bernstein-type bonuses \citep{bourel2020tightening,fruit2020improved}.

In the episodic finite-horizon setting considered in this paper, the optimal worst-case scaling was established by the UCBVI algorithm of \citet{azar2017minimax}, which attains regret $\widetilde{\mathcal O}\!\big(H\sqrt{SAK} + H\sqrt{KH}\big)$ and derives information-theoretic lower bounds of $\widetilde{\mathcal O}\!\big(H\sqrt{SAK}\big)$.  
A number of works have since refined or extended this result: \citet{efroni2019tight} proposed the EULER algorithm, which yields worst-case bounds with improved dependence on the horizon; see also the recent minimax analyses and extensions to factored MDPs, e.g., \citet{tian2020towards}.  
These results, however, are primarily worst-case and focus on bounding either the expectation or a single high-probability quantile of the regret, without characterizing the entire tail behavior.

\paragraph{Instance-dependent and gap-dependent analysis.}
Motivated by the classical gap-dependent theory for stochastic bandits (see, e.g., the survey of \citet{bubeck2012regret}), there has been a surge of interest in instance-dependent regret bounds for tabular MDPs.  
\citet{simchowitz2019non} derived the first non-asymptotic gap-dependent regret bounds for model-based tabular RL, showing that the regret can scale polylogarithmically in $K$ (up to polynomial factors in $(S,A,H)$ and $1/\gapstar$).  
\citet{zanette2019tighter} introduces the environmental norm, a variance-like quantity that controls the difficulty of exploration, and designed an optimistic algorithm whose regret depends on this norm rather than only on worst-case parameters.  
More recently, \citet{dann2021beyond} developed improved instance-dependent bounds that replace the global value-function gaps with more local notions tailored to the subset of states visited by optimal policies, together with new information-theoretic lower bounds showing the necessity of these refined gap notions. On the model-free side, \citet{yang2021q} showed that optimistic $Q$-learning enjoys logarithmic regret in episodic tabular MDPs under a global suboptimality-gap on $Q^\star$, thus establishing that simple temporal-difference methods can match the best-possible gap-dependent rates of model-based algorithms. \citet{xu2021fine} proposed an adaptive multi-step bootstrapping scheme that yields tighter instance-dependent bounds.

Under the linear MDP assumption with function approximation, \citet{he2021logarithmic} establishes instance-dependent logarithmic regret for LSVI-UCB and UCRL-VTR in linear and linear-mixture MDPs. \citet{velegkas2022reinforcement} further refine this picture by designing algorithms that incur only a bounded number of policy switches. Most recently, \citet{zhao2025logarithmic} extend these guarantees to adversarial environments.

\paragraph{Tail behavior and risk-sensitive objectives.}
A recent line of work, starting from \citet{fan2024fragility}, shows that bandit algorithms that are information-theoretically optimized for Lai--Robbins optimal expected regret can exhibit extremely heavy-tailed regret distributions: the regret tail behaves like a truncated Cauchy, and higher moments grow polynomially in the horizon. Building on this observation, \citet{simchi2022simple,simchi2023stochastic} design policies that explicitly trade off optimality and tail risk, achieving worst-case optimal (or near-optimal) expected regret while simultaneously enforcing exponentially light tails for large-regret events and characterizing sharp Pareto frontiers between expected regret and tail probabilities.

In contrast, RL work that explicitly targets tail or risk-sensitive criteria usually changes the performance objective rather than analyzing the tail of classical regret.  
For example, there is a growing literature on risk-sensitive or distributional RL that optimizes mean-variance, CVaR, or more general optimized certainty equivalent (OCE) criteria in MDPs; see, e.g., \citet{tamar2015optimizing,chow2015risk,xu2023regret} and references therein.  
These works derive sample-complexity or regret guarantees for the risk-sensitive value of a policy, not for the tail of the standard regret with respect to the risk-neutral optimal policy.  

Our work is complementary to this risk-sensitive line: we keep the classical regret objective but ask for an instance-dependent characterization of its tail distribution.  
In particular, given a global $Q^\star$-gap, we show that a UCBVI-type algorithm enjoys sub-Gaussian and sub-Weibull tail bounds for the cumulative regret, together with instance-dependent control of its expectation.  
To the best of our knowledge, this is the first result that treats regret concentration in episodic RL.

\section{Problem Formulation}

We consider an MDP $(\mathcal{S},\mathcal{A},\pp^*,r,H)$ with finite horizon $H$, where $\mathcal S$ and $\mathcal A$ are finite state and action spaces with $S \coloneqq |\mathcal{S}|$ and $A \coloneqq |\mathcal{A}|$, the transition kernel is $\pp^* = \{P_h^*\}_{h=0}^{H-2}$, and the reward function is $r = \{r_h\}_{h=0}^{H-1}$, where each stage-$h$ reward is given by $r_h : \mathcal S \times \mathcal A \to [0,1]$. 

A deterministic Markov policy is a collection $\pii=\{\pi_h\}_{h=0}^{H-1}$ with $\pi_h:\mathcal S\to\mathcal A$. Each episode starts at a fixed initial state $s_0\in\mathcal S$ and lasts exactly $H$ steps. Given any state $s_h$ at horizon $h$, following policy $\pii$, we take the action $a_h=\pi_h(s_h)$ and transition to the next random state $s_{h+1}\sim P_h^*(\cdot\mid s_h,a_h)$. The value and action-value functions are defined as
\begin{align*}
V_{h}^{\pii}(s)&\;=\;\mathbb E_{\pii,\pp^*}\!\Big[\sum_{t=h}^{H-1} r_t (s_t, a_t)\,\Big|\,s_h=s\Big],\\
Q_{h}^{\pii}(s,a)
&\;=\; r_h(s,a) +\sum_{s'\in\mathcal S} P_h^*(s'\mid s,a)\,V_{h+1}^{\pii}(s'),
\end{align*}
with the convention $V_H^{\pii}\equiv 0$. Here $\E_{\pii,\pp^*}[\cdot]$ denotes expectation with respect to policy $\pii$ and transition kernel $\pp^*$. The optimal value and action-value functions are
\[
V_h^*(s)\coloneqq\max_{\pii}V_h^{\pii}(s),
\qquad
Q_h^*(s,a)\coloneqq r_h(s,a)+\sum_{s'\in\mathcal S} P_h^*(s'\mid s,a)\,V_{h+1}^*(s'),
\]
and the set of optimal actions at $(s,h)$ is
\[
\mathcal A_h^*(s)\coloneqq\arg\max_{a\in\mathcal A} Q_h^*(s,a).
\]
Our analysis is instance-dependent and exploits a global $Q^*$-gap:
\begin{align*}
\gap_h(s,a) \;&\coloneqq\; \Big(V_h^*(s)-Q_h^*(s,a)\Big),\qquad \forall (s,h)\in\mathcal S\times[H], \\
\gapstar \;&\coloneqq\; \min_{\substack{(s,h)\in\mathcal S\times[H]\\ a\notin \mathcal A_h^*(s)}} \gap_h(s,a) .\numberthis\label{eq:gap}
\end{align*}

Interaction proceeds over $K$ episodes indexed by $k\in[K]$\footnote{Throughout, $[n]=\{0,1,\dots,n-1\}$ and $[n_1,n_2] = \{n_1,n_1+1,\dots,n_2-1\}$.}. At the start of episode $k$, the learner commits to a Markov policy $\pii^k=\{\pi_h^k\}_{h=0}^{H-1}$, then observes a trajectory
\[
s_0,\ a_0^k,\ s_1^k,\ a_1^k,\ \dots,\ s_{H-1}^k,\ a_{H-1}^k,
\]
where $s_0^k=s_0$, $a_h^k=\pi_h^k(s_h^k)$, and $s_{h+1}^k\sim P_h^*(\cdot\mid s_h^k,a_h^k)$ for all $h\in\{0,\dots,H-2\}$. The performance of the learner is measured by the cumulative regret
\begin{equation}
R_K \;=\; \sum_{k=0}^{K-1}\Big(V_0^*(s_0)-V_0^{\pii^k}(s_0)\Big).\nonumber
\end{equation}
Our goal is to develop an algorithm that attains low expected cumulative regret and, simultaneously, exhibits a desirable tail distribution for $R_K$.

\section{Main Results}
In this section, we present our main results. First, in Section~\ref{sec:MB}, we study the tail behavior of the regret for a UCBVI-type algorithm in the model-based setting. Next, in Section~\ref{sec:MF}, we analyze the regret tail of $Q$-learning as a representative model-free algorithm.

\subsection{Model-based tail bound: UCBVI algorithm} \label{sec:MB}
We analyze a UCBVI-type algorithm that uses the following bonus function:
\[
b^k_h(n)
=
\begin{cases}
(H-h-1)\sqrt{\dfrac{0.5S\ln 2+\mu K^\alpha}{n}}, & \text{\KD},\\[2ex]
(H-h-1)\sqrt{\dfrac{0.5S\ln 2+\mu (k+1)^\alpha}{n}}, & \text{\KI},
\end{cases}
\qquad
b^k_h(0)\coloneqq\infty.
\]
Here $k\in\{0,\dots,K-1\}$ denotes the episode index, $n\ge 1$ denotes the number of visits to the corresponding state–action pair, and $\alpha\in[0,1]$ and $\mu>0$ are tunable parameters. \text{\KD} refers to the bonus used in the $K$-dependent algorithm, and \text{\KI} refers to the bonus used in the $K$-independent algorithm.

The term $K^\alpha$ in the definition of $b_h^k(\cdot)$ plays the role of a global exploration budget that directly controls the aggressiveness of the bonuses. In the original minimax UCBVI analysis of \citet{azar2017minimax}, as well as in the subsequent instance-dependent analysis of \citet{simchowitz2019non}, the exploration terms scale only logarithmically with the time horizon, roughly as $\sqrt{{\log(KSAH)}/{n}}$, which is tailored to control the expected regret. In contrast, our bonuses scale as $\sqrt{(S\ln 2 + \mu K^\alpha)/n}$ (or $\sqrt{(S\ln 2 + \mu (k+1)^\alpha)/n}$ in the $K$-independent case), so that the uncertainty radius around the empirical model grows polynomially in $K$. This more conservative choice incentivizes sustained exploration even in later episodes: when the algorithm has visited a state--action pair only moderately many times, the bonus remains substantial and pushes the optimistic value estimates upward. From a tail perspective, this is precisely what allows us to derive nontrivial bounds on $\PP(R_K \ge x)$ for large thresholds $x$, at the price of a controlled polynomial dependence on $K$ inside the bonus which will affect the expected regret.

\begin{algorithm}[ht]
\caption{Optimistic Value Iteration}
\label{algo:UCBVI-gap}
\begin{algorithmic}[1]
    \STATE \textbf{Input:} $H,\mathcal S,\mathcal A$, rewards $r_h:\mathcal{S}\times\mathcal{A}\to[0,1]$, initial state $s_0$, bonus schedules $\{b^k_h(\cdot)\}$.
\STATE \textbf{Init:} For all $(s,a,h)$, $n_h^0(s,a)=0$ and $\hat P_h^0(\cdot\mid s,a)=\tfrac1S$, and $\hat P_{H-1}^k\equiv 0$ and all $k$.
\FOR{$k=0,1,\dots,K-1$}
\STATE \textbf{Planning:} Set $V_H^k\equiv 0$. For $h=H-1,\dots,0$ and each $s$:
\begin{align*}
Q_h^k(s,a)&\coloneqq\min\left\{r_h(s,a)+\sum_{s'}\hat P_h^k(s'\mid s,a)\,V_{h+1}^k(s')+b_h^k(n_h^k(s,a)),H-h\right\},
\quad \forall a\in\mathcal A,\\
V_h^k(s)&\coloneqq\max_{a\in\mathcal A}Q_h^k(s,a),\qquad
\pi_h^k(s)\in\arg\max_{a}Q_h^k(s,a).
\end{align*}
\STATE \textbf{Execution:} Roll out $\pii^k=\{\pi_h^k\}_{h=0}^{H-1}$ from $s_0$ under $\pp^*$:
\begin{align*}
  &a_h^k=\pi_h^k(s_h^k)\quad \text{and} \quad s_{h+1}^k\sim P_h^*(\cdot\mid s_h^k,a_h^k),\qquad \forall h\in[H-1] , \qquad a_{H-1}^k=\pi_{H-1}^k(s_{H-1}^k).
\end{align*}
\STATE \textbf{Update:} For all $(s,a,h)$:
\begin{align*}
&n_h^{k+1}(s,a,s')=\sum_{i=0}^{k}\mathbbm 1\{(s_h^i,a_h^i,s_{h+1}^i)=(s,a,s')\},\quad
n_h^{k+1}(s,a)=\sum_{s'}n_h^{k+1}(s,a,s'),\\
&\hat P_h^{k+1}(s'\mid s,a)=
\begin{cases}
\dfrac{n_h^{k+1}(s,a,s')}{n_h^{k+1}(s,a)}, & n_h^{k+1}(s,a)\ge1,\\[1ex]
\dfrac1S, & n_h^{k+1}(s,a)=0.
\end{cases}
\end{align*}
\ENDFOR
\end{algorithmic}
\end{algorithm}

The UCBVI algorithm with the above bonus function is presented in Algorithm~\ref{algo:UCBVI-gap}. The input consists of the finite-horizon MDP parameters $(H,\mathcal S,\mathcal A,r)$, the initial state $s_0$, and the family of bonus schedules $\{b_h^k(\cdot)\}$. In the initialization step, the algorithm sets all visit counts $n_h^0(s,a)$ to zero and uses the uniform distribution $\hat P_h^0(\cdot\mid s,a)=\tfrac1S$ as an initial model of the dynamics.

At the beginning of each episode $k$, the algorithm performs a planning phase: it runs a backward dynamic-programming pass from $h=H-1$ down to $h=0$, maintaining an optimistic value function $V_h^k$ and $Q$-function $Q_h^k$. For each state-action pair $(s,a)$ at horizon $h$, it forms the optimistic estimate
\[
Q_h^k(s,a)
=
\min\left\{r_h(s,a)+\sum_{s'}\hat P_h^k(s'\mid s,a)\,V_{h+1}^k(s')+b_h^k\big(n_h^k(s,a)\big),\,H-h\right\},
\]
where the bonus $b_h^k(n_h^k(s,a))$ inflates the value of actions whose transition probabilities are still poorly estimated. The value function $V_h^k(s)$ is then defined using the greedy policy $\pi_h^k(s)$ which selects any maximizer of $Q_h^k(s,\cdot)$. This can be viewed as solving the backward dynamic equations with the estimated transition kernel $\hat{\pp}^k$ and the modified reward function $r_h(\cdot,\cdot) + b_h^k\big(n_h^k(\cdot,\cdot)\big).$

In the execution phase, the agent rolls out the greedy policy $\pi^k$ for one episode starting from $s_0$, interacting with the true MDP $\pp^*$ and collecting a trajectory $\{(s_h^k,a_h^k,s_{h+1}^k)\}_{h=0}^{H-1}$. Finally, in the update phase, the algorithm increments the visit counts $n_h^{k+1}(s,a,s')$ and $n_h^{k+1}(s,a)$ along the realized trajectory and updates the empirical transition kernel $\hat P_h^{k+1}(\cdot\mid s,a)$ by normalized counts. The procedure then repeats with the updated model and bonuses in the next episode.

Theorem~\ref{thm:tailGap-UCBVI} characterizes the tail behavior of the regret of Algorithm~\ref{algo:UCBVI-gap}. In this theorem, we use the notation $(x)_+ \coloneqq \max\{x,0\}$.

\begin{theorem}
\label{thm:tailGap-UCBVI}
Fix $\alpha\in[0,1]$ and $\mu>0$. For all $\gamma \in [0,1]$, the tail of the cumulative regret $R_K$ of Algorithm \ref{algo:UCBVI-gap} satisfies
\[
\PP(R_K\ge x)
\ \le\
\exp \left( - \frac{(x-\lceil K^\gamma\rceil H - m_K^{\mathrm{mb}})^2}{4H(H+1)(2H+1)K/3}\right)
\;+\;\delta_K^{\mathrm{mb}}(\gamma), \qquad \forall x\geq H\lceil K^\gamma\rceil,
\]
where
\begin{equation}
\label{eq:mk-gap-UCBVI}
m_K^{\mathrm{mb}}
\;\coloneqq\frac{128 H^3 S A (0.5 S \ln 2 + \mu K^\alpha)}{\gapstar},
\end{equation} 
and 
\begin{align}\label{eq:delta_K_K}
    \delta_K^{\mathrm{mb}}(\gamma) = S A H K
    \begin{cases}
        \displaystyle \exp\left(- 2\mu K^{\alpha} \right),
        & \text{\KD}, \\[2ex]
        \displaystyle \exp\left(-2\mu K^{\gamma\alpha}\right),
        & \text{\KI}.
    \end{cases} 
\end{align}
\end{theorem}

Theorem~\ref{thm:tailGap-UCBVI} shows that, under a global $Q^\star$-gap, as defined in \eqref{eq:gap}, the tail of the cumulative regret $R_K$ admits a decomposition into a sub-Gaussian component and a residual term $\delta_K^{\mathrm{mb}}(\gamma)$. The quantity $m_K^{\mathrm{mb}}$ is the baseline instance-dependent level of regret around which the distribution concentrates: it aggregates the bonuses incurred before all suboptimal actions have been sufficiently explored and eliminated, and it is of order
\begin{align*}
    m_K^{\mathrm{mb}} \;=\; \mathcal O\!\big(H^3 SA(S+K^\alpha)/\gapstar\big).
\end{align*} 
For thresholds $x$ that are moderately larger than $m_K^{\mathrm{mb}}$, the bound
$\PP(R_K \ge x)$ is of sub-Gaussian form with variance proxy $H(H+1)(2H+1)K/3$, reflecting the martingale structure of the regret fluctuations after the main exploration phase.

\begin{prop}
\label{prop:expectGap-UCBVI}
Under the same setting as in Theorem \ref{thm:tailGap-UCBVI}, we have
\[
\mathbb E[R_K]
\ \le\
H\lceil K^\gamma\rceil + m_K^{\mathrm{mb}} + KH(H+1)\delta_K^{\mathrm{mb}}(\gamma),
\]
where $m_K^{\mathrm{mb}}$ is given in \eqref{eq:mk-gap-UCBVI} and $\delta_K^{\mathrm{mb}}(\gamma)$ in \eqref{eq:delta_K_K}. 
\end{prop}
Proposition~\ref{prop:expectGap-UCBVI} obtains an instance-dependent upper bound on the expected regret. The first two terms, $m_K^{\mathrm{mb}} + H\lceil K^\gamma\rceil$, capture the contribution of the baseline instance-dependent level of regret and the initial ``burn-in'' phase. The final term, proportional to $ KH(H+1)\delta_K^{\mathrm{mb}}(\gamma)$, accounts for the rare events in which the empirical model falls outside the ``good'' set at some episodes. In such events the estimated transition kernel $\hat{\pp}^k$ departs substantially from the true kernel $\pp^*$ for some $k\in[K]$. Consequently, we cannot ensure that $V^k_0(s_0)\ge V^*_0(s_0)$ uniformly over $k$, i.e., the \emph{optimism} property may fail. Thanks to the exponential decay of $\delta_K^{\mathrm{mb}}(\gamma)$ in $K^\alpha$ (or $K^{\gamma\alpha}$), this residual contribution is negligible: for any fixed $\alpha>0$ and suitable choice of the constant $\mu > 0$, the product $K\delta_K^{\mathrm{mb}}(\gamma)$ vanishes.

The tuning parameter $\alpha\in[0,1]$ controls the tradeoff between the expected regret and the decay rate of the residual term $\delta_K^{\mathrm{mb}}(\gamma)$. Larger values of $\alpha$ make the bonuses more aggressive, which increases the upper bound on the $\mathbb E[R_K]$ but also causes $\delta_K^{\mathrm{mb}}(\gamma)$ to decay faster with $K$ (either as $\exp(-2\mu K^\alpha)$ or as $\exp(-2\mu K^{\gamma\alpha})$ depending on the bonus schedule). In this sense, $\alpha$ interpolates between prioritizing a smaller expected regret ($\alpha$ small) and enforcing lighter tails for very large deviations ($\alpha$ closer to one). By contrast, the parameter $\gamma\in[0,1]$ is not a tuning knob of the algorithm: it appears only in the analysis and controls the length of the initial burn-in phase $K^\gamma$ during which we make no attempt to exploit tail concentration. In Corollary~\ref{cor:tailExpBound-UCBVI}, we choose $\gamma$ as a function of $x$ to simplify the resulting tail bound into a clean two-regime description.

\begin{cor}\label{cor:tailExpBound-UCBVI}
 Fix $\alpha\in[0,1]$. For the \text{\KD} algorithm, we have
    \begin{align*}
    \PP(R_K\geq x)\leq\begin{cases}
        \exp\left(-\tilde{\mathcal{O}}\left(\frac{x^2}{H^3K}\right)\right)  & \text{if } \mathcal{O}(H^3SA(S+K^\alpha)/\gapstar)\leq x \leq \tilde{\mathcal{O}}(H^{3/2}K^{(1+\alpha)/2})\\
        \exp\left(-\tilde{\mathcal{O}}\left( K^\alpha \right)\right) & \text{if }  x \geq \tilde{\mathcal{O}}(H^{3/2}K^{(1+\alpha)/2})
    \end{cases}
    \end{align*}
    For \text{\KI} algorithm, we have
    \begin{align*}
    \PP(R_K\geq x)\leq\begin{cases}
        \exp\left(-\tilde{\mathcal{O}}\left(\frac{x^2}{H^3K}\right)\right)  & \text{if } \mathcal{O}(H^3SA(S+K^\alpha)/\gapstar)\leq x \leq \tilde{\mathcal{O}}(H^{(3-\alpha)/(2-\alpha)}K^{1/(2-\alpha)})\\
        \exp\left(-\tilde{\mathcal{O}}\left( \frac{x^\alpha}{H^\alpha} \right)\right) & \text{if }  x \geq \tilde{\mathcal{O}}(H^{(3-\alpha)/(2-\alpha)}K^{1/(2-\alpha)})
    \end{cases}
    \end{align*}
    Furthermore, $\E[R_K] \leq \mathcal{O}(H^3SA(S+K^\alpha)/\gapstar)$.
\end{cor}
\begin{proof}[Proof of Corollary \ref{cor:tailExpBound-UCBVI}]
    The proof follows by setting $\gamma = 0$ for \text{\KD} algorithm, and setting $\gamma=\log_K(x/H)$ for \text{\KI} algorithm.
\end{proof}

It is natural to compare Corollary~\ref{cor:tailExpBound-UCBVI} with the existing gap-dependent bounds for tabular MDPs. For model-based algorithms, \citet{simchowitz2019non} show that,  given a global $Q^\star$-gap, for any fixed confidence level $\delta \in (0,1)$, the cumulative regret satisfies the following high-probability bound:
\[
\PP\left(R_K \ge \tilde{\mathcal O}\!\Big(H^4SA(S\vee H) + 
\frac{H^3SA}{\gapstar}
\log (K/\delta)
\Big)\right)
\;\le\;
\delta.
\]
However, they do not characterize the tail behavior $\PP(R_K\ge x)$ for all $x$. Setting $\delta=1/K$, this high probability bound implies the expected regret that scale polylogarithmically in $K$, with bounds of the form
\[
\mathbb E[R_K]
\;\le\;
\tilde{\mathcal O}\!\Big(
\frac{H^3SA}{\gapstar}
\log K
\Big).
\]
In terms of the MDP parameters, this matches the expected regret upper bound presented in Proposition \ref{prop:expectGap-UCBVI}. The primary difference is the refined dependence on $K$, which results from a more aggressive exploration strategy. Similar instance-dependent rates have been obtained for optimistic $Q$-learning and related model-free methods~\citep{yang2021q,velegkas2022reinforcement}. All the algorithms presented in these works depend on the specific choice of confidence level $\delta\in(0,1)$.

In contrast, our results are instance-dependent in a comparable regime, but they are organized around the \emph{entire} tail rather than a single quantile. 
Instead of working at a fixed confidence level $\delta$ and producing a bound, Corollary~\ref{cor:tailExpBound-UCBVI} describes how the entire curve $x\mapsto\PP(R_K\ge x)$ behaves:
for $x$ above the instance-dependent baseline $m_K^{\mathrm{mb}}$ we obtain a sub-Gaussian tail
\[
\PP(R_K\ge x)
\;\le\;
\exp\!\left(-\tilde{\mathcal O}\left(\frac{x^2}{H^3 K}\right)\right),
\]
up to a transition scale that depends on $(H,K,\alpha)$, and beyond that scale the tail becomes sub-Weibull, with decay of order $\exp\{-\tilde{\mathcal O}(K^\alpha)\}$ in the $\KD$ case or $\exp\{-\tilde{\mathcal O}(x^\alpha/H^\alpha)\}$ in the $\KI$ case.

\subsection{Model-free tail bounds: optimistic $Q$-learning}\label{sec:MF}
We analyze the tail behavior of an optimistic $Q$-learning algorithm studied in \citep{jin2018q, yang2021q}.

We consider the $K$-dependent (\text{\KD}) settings, with bonus functions defined as
\[
    b_h(n) = 4 H^{1.5} \sqrt{\dfrac{\ln(2SAHK)+\mu K^\alpha}{n}}
\]
Here $n\ge 1$ is the visit count of the associated state--action pair, and $\alpha\in[0,1]$ and $\mu>0$ are tuning parameters.

In contrast to the UCBVI algorithm, the update equations in optimistic $Q$-learning depend on the choice of step sizes.
We define the step-size sequence and its unrolled weights as
\begin{align}\label{eq:alpha-weights}
\lr^{(t)} \;\coloneqq\; \frac{H+1}{H+t+1},\qquad  \lr^{t\rightarrow t'} \;\coloneqq\; \lr^{(t)} \prod_{j=t+1}^{t'} (1-\lr^{(j)})
,\qquad t\ge 0.
\end{align}
Notice that $\sum_{t=0}^{t'} \lr^{t\to t'}=1$, with the convention that $\prod_{j=t'+1}^{t'} \equiv 1$.

\begin{algorithm}[t]
\caption{Optimistic Q-learning}
\label{alg:qlearning-tail}
\begin{algorithmic}[1]
\STATE \textbf{Input:} $H,\mathcal S,\mathcal A$, rewards $r_h:\mathcal{S}\times\mathcal{A}\to[0,1]$, initial $s_0$, bonus schedules $\{b_h(\cdot)\}$, step size schedule $\{\lr^{(t)}\}$.
\STATE \textbf{Init:} For all $(s,a,h)$, set $Q_h^0(s,a)= H$ and $V_h^0(s)= H$ and $n_h^0(s,a)= 0$. 
\FOR{$k=0,1,2,\dots,K-1$}
    \STATE \textbf{Planning:} For all $(s,h)$:
\begin{align*}
\pi_h^k(s)\in\arg\max_{a\in\mathcal{A}}Q_h^k(s,a).
\end{align*} \label{alg:planning_line}
    \STATE \textbf{Execution:} Roll out $\pii^k=\{\pi_h^k\}_{h=0}^{H-1}$ from $s_0$ under $\pp^*$:
    \begin{align*}
    a_h^k=\pi_h^k(s_h^k)\quad \text{and} \quad s_{h+1}^k\sim P_h^*(\cdot\mid s_h^k,a_h^k),\qquad \forall h\in[H-1] , \qquad a_{H-1}^k=\pi_{H-1}^k(s_{H-1}^k).
    \end{align*}
    \STATE \textbf{Update:}
    Set $V_{H}^k\equiv 0$.
    For all $h$:  
    \begin{align*}
    &t=n_h^{k+1}(s_h^k,a_h^k)= n_h^k(s_h^k,a_h^k)+1\\
    &Q_h^{k+1}(s_h^k,a_h^k) \leftarrow (1-\lr^{(t)})\,Q_h^k(s_h^k,a_h^k)
    +\lr^{(t)} \Big(r_h(s_h^k,a_h^k)+V_{h+1}^k(s_{h+1}^k)+ b_h(t)\Big),\\
    &V_h^{k+1}(s_h^k) \leftarrow \min\{H,\max_{a\in\mathcal{A}} Q_h^{k+1}(s_h^k,a)\}.
    \end{align*}
\ENDFOR
\end{algorithmic}
\end{algorithm}

Algorithm~\ref{alg:qlearning-tail} maintains optimistic action-value estimates
$\{Q_h^k(s,a)\}$ and state-value estimates $\{V_h^k(s)\}$ for each stage
$h\in\{0,1,\dots,H-1\}$, together with visitation counts $n_h^k(s,a)$.
It is initialized optimistically as
$Q_h^0(s,a)=H$, $V_h^0(s)=H$ for all $(s,a,h)$, and $n_h^0(s,a)=0$.

At the start of episode $k$, the planning step forms a greedy policy
with respect to the current $Q$-values. The execution step then rolls out $\pi^k=\{\pi_h^k\}_{h=0}^{H-1}$
from the fixed initial state $s_0$ in the unknown MDP with transitions
$P_h^*$, generating a trajectory $(s_h^k,a_h^k,s_{h+1}^k)_{h=0}^{H-1}$. After observing the trajectory, the update step performs on-trajectory
temporal-difference updates. First, the algorithm sets the terminal value $V_H^k\equiv 0$.
Then for each stage $h=0,1,\dots,H-1$, it increments the visit count of the
experienced pair $(s_h^k,a_h^k)$, and updates only the corresponding $Q$-entry using step size $\lr^{(t)}\in(0,1]$
toward a bonus-augmented TD target:
\[
Q_h^{k+1}(s_h^k,a_h^k)
=(1-\lr^{(t)})\,Q_h^k(s_h^k,a_h^k)
+\lr^{(t)}\Big(r_h(s_h^k,a_h^k)+V_{h+1}^k(s_{h+1}^k)
+b_h\big(t\big)\Big),
\]
where $t=n_h^{k+1}(s_h^k,a_h^k)$ is the visitation count to state-action $(s,a)$ at horizon $h$ at the end of episode $k$. All other action-values are left unchanged. Finally, the state-value at the visited states are updated to the greedy value and are clipped to the horizon, i.e., $V_h^{k+1}(s_h^k)=\min\{H, \max_{a\in\mathcal{A}} Q_h^{k+1}(s_h^k,a)\}$, while $V_h^{k+1}(s)=V_h^k(s)$ for all other states. The optimism is induced by the initial over-estimates and the nonnegative
bonus $b_h(\cdot)$, decreasing in the visit count to encourage
exploration early and vanish as sampling increases. 

The episode index $k$ in $Q_h^k$ and $V_h^k$ is used only for notational clarity (e.g., in the analysis) to emphasize the evolution of the estimates
across episodes. In an actual implementation, it suffices to store and update
a single table for $\{Q_h(s,a)\}$ and a single table for $\{V_h(s)\}$, updating
their entries in-place as data are collected.

To mirror Theorem~1 in our model-based section, we state a parameterized tail bound indexed by $\alpha\in[0,1]$, with an explicit sub-Gaussian component plus a residual ``bad event'' term.

\begin{theorem}
\label{thm:tailGap-Qlearning}
Fix $\alpha\in[0,1]$ and $\mu>0$. The tail of the cumulative regret $R_K$ of Algorithm \ref{alg:qlearning-tail} satisfies
\begin{align}
\PP\!\Big(R_K \ge x\Big)
\;\le\;
\exp\!\left(
-\frac{\big(x  - m_K^{\mathrm{mf}}\big)_+^2}{H^3K}
\right)
+\delta_K^{\mathrm{mf}}, \qquad \forall x\geq 0,
\label{eq:mf-tail-master}
\end{align}
where the baseline level is
\begin{align}
m_K^{\mathrm{mf}}
\;\coloneqq\;
\frac{3242 eSAH^6(\ln(2SAHK)+\mu K^\alpha)}{\gapstar},
\label{eq:mf-mK}
\end{align}
and the residual term is
\begin{align}
\delta_K^{\mathrm{mf}}
\;\coloneqq\;
\exp\!\big(-\mu K^\alpha\big)
\label{eq:mf-deltaK}
\end{align}
\end{theorem}
Similar to Theorem~\ref{thm:tailGap-UCBVI}, Theorem~\ref{thm:tailGap-Qlearning} decomposes tail of regret into a sub-Gaussian part plus a residual $\delta_K^{\mathrm{mf}}$. Finally, we derive a bound for the expected regret.
\begin{prop}
\label{prop:expectGap-UCBQVI}
Under the same setting as in Theorem \ref{thm:tailGap-Qlearning}, we have
\[
\mathbb E[R_K]
\ \le\
m_K^{\mathrm{mf}} + 2 KH^2\delta_K^{\mathrm{mf}}
\]
where $m_K^{\mathrm{mf}}$ is given in \eqref{eq:mf-mK} and $\delta_K^{\mathrm{mf}}$ in \eqref{eq:mf-deltaK}. 
\end{prop}

\section{Proof sketch}

We outline the proof of Theorem~\ref{thm:tailGap-UCBVI} and highlight where each term in the bound arises. The proof of Theorem~\ref{thm:tailGap-Qlearning} follows similar ideas and is presented in the Appendix.

\paragraph{Burn-in decomposition.}
Split the cumulative regret into an initial ``burn-in'' part and the remainder:
\[
R_K
=
\underbrace{\sum_{k=0}^{\lceil K^\gamma\rceil-1}\!\Big(V_0^*(s_0)-V_0^{\pii^k}(s_0)\Big)}_{R_K^{(0)}}
+
\underbrace{\sum_{k=\lceil K^\gamma\rceil}^{K-1}\!\Big(V_0^*(s_0)-V_0^{\pii^k}(s_0)\Big)}_{R_K^{(1)}}.
\]
Since each episode regret is at most $H$, we have $R_K^{(0)} \le H\lceil K^\gamma\rceil,$ which yields the additive shift $H\lceil K^\gamma\rceil$ in the final tail bound. This initial burn-in term is essential for the analysis of the horizon-independent algorithm. 
By tuning $\gamma$, we trade off $R_K^{(0)}$ and $R_K^{(1)}$, so that the order matches.

\paragraph{Bellman-style telescoping and a martingale decomposition.}
Fix an episode $k\ge \lceil K^\gamma\rceil$ and define the value difference
$\Dvs_h^k(s)\coloneqq V_h^*(s)-V_h^{\pii^k}(s)$.
Along the realized trajectory $(s_h^k,a_h^k)_{h=0}^{H-1}$ one can write
\[
\Dvs_h^k(s_h^k)
=
\gap_h(s_h^k,a_h^k)
+
\E\!\left[\Dvs_{h+1}^k(s_{h+1}^k)\mid \mathcal{F}_h^k\right],
\]
and define the martingale difference term
\[
\Dvsm_{h}^k(s_{h}^k)
\;\coloneqq\;
\Dvs_{h}^k(s_{h}^k)
-\E\!\left[\Dvs_{h}^k(s_{h}^k)\mid \mathcal{F}_{h-1}^k\right].
\]
Summing over $h=0,\dots,H-1$ telescopes (using $\Dvs_H^k\equiv 0$) to give
\begin{equation}\label{eq:sketch-telescope}
V_0^*(s_0)-V_0^{\pii^k}(s_0)
=
\sum_{h=0}^{H-1}\gap_h(s_h^k,a_h^k)
-
\sum_{h=0}^{H-1}\Dvsm_{h+1}^k(s_{h+1}^k).
\end{equation}
We introduce the $\gap$ function because it enables instance-dependent guarantees.

\paragraph{Clipping and the role of the global gap.}
We define the clipping function
\[
\clip(x)=0 \ \text{if } x<\gapstar,\qquad \clip(x)=x \ \text{otherwise}.
\]
By definition of the global gap $\gapstar$, for every $(s,h)$ we have
$\gap_h(s,a)\in\{0\}\cup[\gapstar,\infty)$, hence $\gap_h(s_h^k,a_h^k)=\clip(\gap_h(s_h^k,a_h^k))$.
Thus \eqref{eq:sketch-telescope} becomes
\[
V_0^*(s_0)-V_0^{\pii^k}(s_0)
=
\sum_{h=0}^{H-1}\clip\!\big(\gap_h(s_h^k,a_h^k)\big)
-
\sum_{h=0}^{H-1}\Dvsm_{h+1}^k(s_{h+1}^k).
\]
The purpose of $\clip(\cdot)$ is that it lets us upper bound \emph{only the suboptimal-action contributions}
via bonuses, while treating the remaining fluctuations as martingale noise.

\paragraph{A high-probability ``good event'' implying optimism.}
Define the per-episode good event
\[
\mathcal{G}^k
\;=\;
\bigcap_{h=0}^{H-2}\bigcap_{s,a}
\left\{
\|P_h^*(\cdot\mid s,a)-\hat P_h^k(\cdot\mid s,a)\|_1
\le
\frac{2b_h^k(n_h^k(s,a))}{H-h-1}
\right\}.
\]
On $\mathcal{G}^k$, a standard backward induction yields optimism:
\begin{equation}\label{eq:sketch-optimism}
Q_h^k(s,a)\ge Q_h^*(s,a),\qquad V_h^k(s)\ge V_h^*(s),\qquad \forall (s,a,h).
\end{equation}
Let $\bar{\mathcal{G}}^K\coloneqq \cap_{k=\lceil K^\gamma\rceil}^{K-1}\mathcal{G}^k$ denote the good event over
all post burn-in episodes.

\paragraph{Bounding the bad-event probability.}
Using a union bound over $(s,a,h)$ and over visit counts $n\le K$, and applying
\cite[Theorem 2.1]{weissman2003inequalities} to control
$\|\hat P_h^k(\cdot\mid s,a)-P_h^*(\cdot\mid s,a)\|_1$ gives
\[
\PP\big((\bar{\mathcal{G}}^K)^c\big)
\le
SAHK\cdot
\begin{cases}
\exp(-2\mu K^\alpha), & \KD,\\
\exp(-2\mu K^{\gamma\alpha}), & \KI,
\end{cases}
\;=\;
\delta_K^{\mathrm{mb}}(\gamma).
\]

\paragraph{Regret on the good event $\leq$ deterministic ``bonus budget'' + martingale.}
On $\bar{\mathcal{G}}^K$, optimism \eqref{eq:sketch-optimism} lets us compare the clipped gap
to the algorithmic optimism gap $\Dvk_h^k(s)\coloneqq V_h^k(s)-V_h^{\pii^k}(s)$. More specifically, on $\bar{\mathcal{G}}^K$, $\clip(\gap(s,a))\leq \clip(\Dvk_h^k(s))$. A one-step recursion for $\Dvk_h^k(s_h^k)$ implies
\[
\Dvk_h^k(s_h^k)\mathbbm{1}\{\mathcal{G}^k\}
\;\le\;
2\,b_h^k(n_h^k(s_h^k,a_h^k))\mathbbm{1}\{\mathcal{G}^k\}
+
\E\!\left[\Dvk_{h+1}^k(s_{h+1}^k)\mathbbm{1}\{\mathcal{G}^k\}\mid \mathcal{F}_h^k\right],
\]
and telescoping over $h$ yields an upper bound of the form
\[
\Dvk_h^k(s_h^k)\mathbbm{1}\{\mathcal{G}^k\}
\;\lesssim\;
\sum_{l=h}^{H-1} b_l^k(n_l^k(s_l^k,a_l^k))\mathbbm{1}\{\mathcal{G}^k\}
\;-\;(\text{martingale terms}).
\]
Combining this with the clipping inequality $\clip\!\left(\sum_{i=1}^H a_i\right)\le \sum_{i=1}^H \clip(Ha_i)$ for all $a_i\ge 0$, one obtains a bound of the form
\begin{equation}\label{eq:sketch-main-decomp}
R_K^{(1)}\mathbbm{1}\{\bar{\mathcal{G}}^K\}
\;\le\;
\underbrace{\sum_{k=\lceil K^\gamma\rceil}^{K-1}\sum_{h=0}^{H-1}(h+1)\,
\clip\!\Big(4H\,b_h^k(n_h^k(s_h^k,a_h^k))\Big)}_{\text{bonus budget}}
\;-\;
\underbrace{\sum_{k=\lceil K^\gamma\rceil}^{K-1}\sum_{h=0}^{H-1}\Dvm_{h+1}^k}_{\text{martingale}},
\end{equation}
where $\{\Dvm_{h+1}^k\}$ is a martingale difference sequence.

\paragraph{Bounding the bonus budget by $m_K^{\mathrm{mb}}$ using the gap.}
Because $\clip(x)=0$ for $x<\gapstar$, the summand in the bonus budget is nonzero only while
\[
4H\,b_h^k(n)\;\ge\;\gapstar,
\quad\text{i.e.,}\quad
n\;\leq\;\frac{16H^4(0.5S\ln2+\mu K^\alpha)}{{\gapstar}^2}.
\]
Thus, for each $(s,a,h)$, only the first $\bar n=\lfloor 16H^4(0.5S\ln2+\mu K^\alpha)/{\gapstar}^2\rfloor$ visits contribute,
and summing $\sum_{n=1}^{\bar n}\frac{1}{\sqrt{n}}\le 2\sqrt{\bar n}$ gives
\[
\sum_{k=\lceil K^\gamma\rceil}^{K-1}\sum_{h=0}^{H-1}(h+1)\,
\clip\!\Big(4H\,b_h^k(n_h^k(s_h^k,a_h^k))\Big)
\;\le\;
\frac{32H^6SA(0.5S\ln2+\mu K^\alpha)}{\gapstar}
\;=\;
m_K^{\mathrm{mb}}.
\]

\paragraph{Sub-Gaussian tail from Azuma--Hoeffding.}
The martingale increments $\Dvm_{h+1}^k$ are uniformly bounded by $5H^2$. There are at most $HK$ such increments, hence Azuma--Hoeffding yields
\[
\PP\!\left(
-\sum_{k=\lceil K^\gamma\rceil}^{K-1}\sum_{h=0}^{H-1}\Dvm_{h+1}^k
\ge t
\right)
\le
\exp\!\left(
-\frac{t^2}{50H^5K}
\right).
\]
Setting $t=(x-H\lceil K^\gamma\rceil-m_K^{\mathrm{mb}})_+$ and using
\eqref{eq:sketch-main-decomp} gives the sub-Gaussian term in the theorem.

\paragraph{Combine with the bad event.}
Finally, decompose
\[
\PP(R_K\ge x)
\le
\PP\big((\bar{\mathcal{G}}^K)^c\big)
+
\PP\big(\{R_K\ge x\}\cap \bar{\mathcal{G}}^K\big),
\]
and combining with the above inequalities we get the result in Theorem~\ref{thm:tailGap-UCBVI}.

\section{Proof Analysis}
\subsection{Proof of Theorem \ref{thm:tailGap-UCBVI}}

Recall that $\pp^* = \{P^*_h\}_{h=0}^{H-2}$ denotes the true transition kernel.  At episode $k$, the algorithm computes a greedy policy $\pi^k = \{\pi^k_h\}_{h=0}^{H-1}$ using an optimistic estimate of the optimum value function $\{V^*_h\}_{h=0}^{H-1}$ based on data collected up to episode $k$. Recall that $R_K$ denote the cumulative regret incurred up to episode $K-1$. Then, writing $s_0^k\equiv s_0$, we have:
\begin{align*}
R_K 
&= \sum_{k=0}^{K-1} \left(V_0^{*}(s_0^k) 
   -  V_0^{\pii^k}(s_0^k)\right)
= \underbrace{\sum_{k=0}^{\lceil K^\gamma\rceil-1} \left( V_0^{*}(s_0^k) - V_0^{\pii^k}(s_0^k)\right) }_{R_K^{(0)}}
 + \underbrace{\sum_{k=\lceil K^\gamma\rceil}^{K-1} \left(V_0^{*}(s_0^k) - V_0^{\pii^k}(s_0^k)\right)}_{R_K^{(1)}}.
\end{align*}
Here, $R_K^{(0)} \le \lceil K^\gamma \rceil H$ denotes the initial regret incurred during episodes with insufficient data. We continue by analyzing $R_K^{(1)}$. Throughout this proof, we define the following clipping function:
\begin{align*}
\clip(x) \;=\;
\begin{cases}
0, & \text{if } x < \gapstar,\\[4pt]
x, & \text{otherwise.}
\end{cases}
\end{align*}

The regret associated with any episode $k\geq\ceil{K^\gamma}$ is given by
\begin{align*}
    V_0^{*}(s_0^k) - V_0^{\pii^k}(s_0^k) &= V_0^{*}(s_0^k) - Q_0^{*}(s_0^k,a_0^k) + (Q_0^{*}(s_0^k,a_0^k) - Q_0^{\pii^k}(s_0^k,a_0^k))\\
    &= \gap_0(s_0^k,a_0^k) + \E\left[V_1^{*}(s_1^k) - V_1^{\pii^k}(s_1^k)\big| \mathcal{F}_0^{k} \right]
\end{align*}
where $\mathcal{F}_h^k = \sigma(\{s_0^0, a_0^0, s_1^0, a_1^0, \ldots, s_{H-1}^0, a_{H-1}^0, s_0^1, a_0^1, \ldots, s_0^k, a_0^k, \ldots, s_{h}^k, a_{h}^k\})$ represents the history of states and actions observed by the algorithm up to episode $k$ and step $h$, and $a_h^k = \pi_h^k(s_h^k)$. Define the value difference $\Dvs_h^k(s) \coloneqq V_h^{*}(s) - V_h^{\pii^k}(s)$. Following a similar argument as above, for all $h\in [H]$, we have
\begin{align*}
    \Dvs^k_h(s_h^k) &= \gap_h(s_h^k,a_h^k)+ \E\left[\Dvs^k_{h+1}(s_{h+1}^k)\big| \mathcal{F}_h^{k} \right]\\
    &= \gap_h(s_h^k,a_h^k) + \Dvs^k_{h+1}(s_{h+1}^k) - \Dvsm^k_{h+1}(s_{h+1}^k),
\end{align*}
where the sequence $\{\Dvsm^k_{h}(s_{h}^k)\}_{h=1,k=0}^{H-1,K-1}$ defined as 
\begin{align*}
    \Dvsm^k_{h}(s_{h}^k) = \Dvs^k_{h}(s_{h}^k) - \E\left[\Dvs^k_{h}(s_{h}^k)\big| \mathcal{F}_{h-1}^{k} \right]
\end{align*}
is a martingale difference sequence with respect to the filtration $\{\mathcal{F}_{h-1}^{k}\}_{h=1,k=0}^{H-1,K-1}$. For the sake of notational simplicity, we define $\Dvsm^k_{H}\equiv 0$ for all $k\in[K]$.

Let $\upsilon^k \in \{0,1,\cdots,H-1,H\}$ denote the first horizon at which a non-optimal action is taken at episode $k$, i.e.,
\begin{align*}
    \upsilon^k \coloneqq \inf \{h: a_h^k \notin \mathcal{A}^*_h(s_h^k)\} 
\end{align*}
with the convention that $\inf\emptyset = H$. Notice that for any $h< \upsilon^k $, we have $\gap_h(s_h^k,a_h^k) = 0$. Using a telescopic sum, noting that $\Dvs^k_H\equiv 0$, we obtain
\begin{align}
    \Dvs^k_0(s_0^k) &= \Dvs^k_{\upsilon^k}(s_{\upsilon^k}^k) - \sum_{h=0}^{{\upsilon^k}-1}\Dvsm^k_{h+1}(s_{h+1}^k) \label{eq:Dvs-equality}
\end{align}
Note that $\Dvs^k_{\upsilon^k}(s_{\upsilon^k}^k) = V_{\upsilon^k}^{*}(s_{\upsilon^k}^k) - V_{\upsilon^k}^{k}(s_{\upsilon^k}^k) + V_{\upsilon^k}^{k}(s_{\upsilon^k}^k) -V_{\upsilon^k}^{\pii^k}(s_{\upsilon^k}^k)$. The bonus function $b_h^k(\cdot)$ is chosen so that $V_h^k$ provides an {optimistic estimate} of $V_h^*$, i.e. $V_h^k \geq V_h^*$, as long as the estimated transition kernel $\hat{P}_{h'}^k$ and the true transition kernel $P_{h'}^*$ remain close 
for all horizons $h' \geq h$. This combined with the monotonicity of the $\clip(\cdot)$ function provides an upper bound for the regret at episode $k\geq \lceil K^\gamma \rceil$. We define the ``good event at episode $k$'' as 
\begin{align*}
\mathcal{G}^k =& \bigcap_{h=0}^{H-2}\bigcap_{s,a}\left\{\|P_h^*(\cdot|s,a)-\hat P_h^k(\cdot|s,a)\|_1 \leq \frac{2b_h^k(n_h^k(s,a))}{H-h-1} \right\}.
\end{align*}

Under the event $\mathcal{G}^k$, we claim that $Q_h^k(s,a) \geq  Q_h^*(s,a)$ for all $(s,a,h)\in\mathcal{S}\times\mathcal{A}\times[H]$. We establish this by induction on the horizon. At the terminal horizon, we have $Q_{H-1}^k(s,a) = Q_{H-1}^*(s,a)$ for all $(s,a) \in \mathcal{S} \times \mathcal{A}$. 
Assuming $Q_{h+1}^k(s,a) \ge Q_{h+1}^*(s,a)$ for all $(s,a)$, we now show that the same inequality holds for horizon $h$.
\begin{align*}
    Q_h^k(s,a) - Q_h^*(s,a)
    = \min\Bigg\{ &
    \sum_{s'} \hat P_h^k(s'|s,a)\big(V_{h+1}^k(s') - V_{h+1}^*(s')\big) \\
    &\quad+ \sum_{s'} \big(\hat P_h^k(s'|s,a) - P_h^*(s'|s,a)\big)V_{h+1}^*(s') 
    + b_h^k\big(n_h^k(s,a)\big), \\
    &\quad H - h - Q_h^*(s,a)\Bigg\}.
\end{align*}

The first term is nonnegative by induction hypothesis, while the second term is bounded as follows:
\begin{align*}
    &\sum_{s'} \!\left(\hat P_h^k(s'|s,a) - P_h^*(s'|s,a)\right)V_{h+1}^*(s')\\
    &\qquad = \sum_{s'} \!\left(\hat P_h^k(s'|s,a) - P_h^*(s'|s,a)\right)
        \left(V_{h+1}^{*}(s') - c\mathbf{1}\right) \\
    &\qquad \ge -\|\hat P_h^k(\cdot|s,a) - P_h^*(\cdot|s,a)\|_1 
       \,\|V_{h+1}^{*} - c\mathbf{1}\|_\infty \tag{Hölder's inequality}\\
    &\qquad = -\tfrac{1}{2}\,\mathrm{span}\big(V_{h+1}^{*}\big)
       \,\|\hat P_h^k(\cdot|s,a) - P_h^*(\cdot|s,a)\|_1 \\
    &\qquad \ge -b_h^k(n_h^k(s,a))\tag{by definition of ${\mathcal{G}}^k$},
\end{align*}
where $c = \tfrac{1}{2}\big(\max_x V_{h+1}^{*}(x) + \min_x V_{h+1}^{*}(x)\big),$ $\mathrm{span}(f) = \max_x f(x) - \min_x f(x)$, and $\mathbf{1}$ denotes the all-ones vector. Hence, under $\bar{\mathcal{G}}^K$, taking maxima and recursing yields
\begin{align}    
    Q_h^k(s,a) \geq Q_h^*(s,a),
    \qquad 
    V_h^k(s)\ge V_h^*(s),\label{eq:good_event}
\end{align}
for all $(s,a,h)\in\mathcal{S}\times\mathcal{A}\times[H]$. 

Defining $\bar{\mathcal{G}}^K \coloneqq\cap_{k=\lceil K^\gamma\rceil}^{K-1}\mathcal{G}^k$, we next bound the probability of the bad event $\left(\bar{\mathcal{G}}^K\right)^\mathrm{c}$:
\begin{align*}
    &\PP \left(\left(\bar{\mathcal{G}}^K\right)^\mathrm{c}\right) \\
    &= \PP\left(\exists s,a,h,k \in [\lceil K^\gamma\rceil,K ]:
        \|P_h^*(\cdot|s,a)-\hat P_h^k(\cdot|s,a)\|_1 
        >  \frac{2b_h^k(n_h^k(s,a))}{H-h-1}\right) \\
    &\leq \sum_{s,a,h} 
        \PP\left(\exists k \in [\lceil K^\gamma\rceil,K ]:
        \sum_{s'} \left| P_h^*(s'|s,a)-
            \frac{n_h^k(s,a,s')}{n_h^k(s,a)}\right| 
        >  \frac{2b_h^k(n_h^k(s,a))}{H-h-1}\right) \tag{union bound}\\
    &= \sum_{s,a,h} 
        \PP\left(\exists k \in [\lceil K^\gamma\rceil,K ]:
        \sum_{s'} \left| P_h^*(s'|s,a)-
            \frac{\sum_{i=0}^{k-1}\mathbbm{1}\{(s^{i}_h,a^{i}_h,s^{i}_{h+1})=(s,a,s')\}}
                 {n_h^k(s,a)}\right| 
        >  \frac{2b_h^k(n_h^k(s,a))}{H-h-1}\right) \\
    &\overset{(a)}{\leq} \sum_{s,a,h} 
        \PP\left(
        \exists n \in \{1,2,\ldots,K\}:
        \sum_{s'} \left| P_h^*(s'|s,a)-
            \frac{\sum_{i=1}^{n}\mathbbm{1}\{X_i(s,a)=s'\}}{n}\right| 
        >  \frac{2b_h^{n \vee \lceil K^\gamma\rceil}(n)}{H-h-1} \right) \\
    &\leq \sum_{s,a,h}\sum_{n= 1}^K 
        \PP\left(
        \sum_{s'} \left| P_h^*(s'|s,a)-
            \frac{\sum_{i=1}^{n}\mathbbm{1}\{X_i(s,a)=s'\}}{n}\right| 
        >  \frac{2b_h^{n \vee \lceil K^\gamma\rceil}(n)}{H-h-1} \right) \tag{union bound}\\
    &\leq \sum_{s,a,h}\sum_{n=1}^K 
        2^{S}\exp\left(-  \frac{n}{2}  \left(\frac{2b_h^{n \vee \lceil K^\gamma\rceil}(n)}{H-h-1} \right)^2\right)
        \tag{\cite[Theorem 2.1]{weissman2003inequalities}}\\
    & \leq S A H K
    \begin{cases}
        \displaystyle \
           \exp\left(- 2\mu  K^{\alpha}\right),
        & \text{\KD}, \\[2ex]
        \displaystyle \ \exp\left(-2\mu K^{\gamma\alpha}\right),
        & \text{\KI}.
    \end{cases} \eqqcolon \delta_K^{\mathrm{mb}}(\gamma) \tag{Lemma \ref{lem:sumUpperBound-UCBVI}}
\end{align*}
where in step (a), $X_i(s,a)\overset{\text{i.i.d.}}{\sim} P_h^*(\cdot \mid s,a)$, and the inequality follows from the facts that $n_h^k(s,a) \leq K$, a coupling argument, and that for any $n\leq k$ with $k \geq K^\gamma$ we have $b_h^{n \vee \lceil K^\gamma\rceil}(n) < b_h^k(n)$.

Decomposing the tail probability of the regret into the contribution of the bad event and its complement, we have
\begin{align*}
    \PP(R_K\geq x) &\leq \PP(R_K^{(1)} \geq x-\lceil K^\gamma\rceil H)\\
    &\leq \PP\left(\left(\bar{\mathcal{G}}^K\right)^\mathrm{c}\right) + \PP\left( \left\{R_K^{(1)}\geq x-\lceil K^\gamma\rceil H\right\}\cap \bar{\mathcal{G}}^K\right)\\
    &\leq \delta_K^{\mathrm{mb}}(\gamma) + \PP\left( \left\{R_K^{(1)}\geq x-\lceil K^\gamma\rceil H\right\}\cap \bar{\mathcal{G}}^K\right)\\
    &= \delta_K^{\mathrm{mb}}(\gamma) + \PP\left( R_K^{(1)} \mathbbm{1}\{\bar{\mathcal{G}}^K\}\geq x-\lceil K^\gamma\rceil H \right). \tag{since $x > \lceil K^\gamma\rceil H$}
\end{align*}

We proceed with bounding the cumulative regret $R_K^{(1)} = \sum_{k=\lceil K^\gamma\rceil}^{K-1} \left(V_0^{*}(s_0^k) - V_0^{\pii_k}(s_0^k)\right)$, under the event $\bar{\mathcal{G}}^K$.
\begin{align*}
    R_K^{(1)} \mathbbm{1}\{\bar{\mathcal{G}}^K\} 
    &= \sum_{k=\lceil K^\gamma\rceil}^{K-1}\left[\Dvs^k_{\upsilon^k}(s_{\upsilon^k}^k) - \sum_{h=0}^{{\upsilon^k}-1}\Dvsm^k_{h+1}(s_{h+1}^k)\right]\mathbbm{1}\{\bar{\mathcal{G}}^K\} \tag{by \eqref{eq:Dvs-equality}}\\
    &=\sum_{k=\lceil K^\gamma\rceil}^{K-1}\left[ V_{\upsilon^k}^{*}(s_{\upsilon^k}^k) \pm V_{\upsilon^k}^{k}(s_{\upsilon^k}^k) -V_{\upsilon^k}^{\pii^k}(s_{\upsilon^k}^k) - \sum_{h=0}^{{\upsilon^k}-1}\Dvsm^k_{h+1}(s_{h+1}^k)\right]\mathbbm{1}\{\bar{\mathcal{G}}^K\}\\
    &=\sum_{k=\lceil K^\gamma\rceil}^{K-1}\left[ \left(V_{\upsilon^k}^{*}(s_{\upsilon^k}^k) \pm V_{\upsilon^k}^{k}(s_{\upsilon^k}^k) -V_{\upsilon^k}^{\pii^k}(s_{\upsilon^k}^k)\right)\mathbbm{1}\{\mathcal{G}^k\} - \sum_{h=0}^{{\upsilon^k}-1}\Dvsm^k_{h+1}(s_{h+1}^k)\right]\mathbbm{1}\{\bar{\mathcal{G}}^K\}\\
    &\leq\sum_{k=\lceil K^\gamma\rceil}^{K-1}\left[ \left(V_{\upsilon^k}^{k}(s_{\upsilon^k}^k) -V_{\upsilon^k}^{\pii^k}(s_{\upsilon^k}^k)\right) \mathbbm{1}\{\mathcal{G}^k\} - \sum_{h=0}^{{\upsilon^k}-1}\Dvsm^k_{h+1}(s_{h+1}^k)\right]\mathbbm{1}\{\bar{\mathcal{G}}^K\}\tag{by \eqref{eq:good_event}}\\
    &=\sum_{k=\lceil K^\gamma\rceil}^{K-1}\left[ \Dvk_{\upsilon^k}^k(s_{\upsilon^k}^k) \mathbbm{1}\{\mathcal{G}^k\} - \sum_{h=0}^{{\upsilon^k}-1}\Dvsm^k_{h+1}(s_{h+1}^k)\right]\mathbbm{1}\{\bar{\mathcal{G}}^K\},\numberthis\label{eq:R_K_1_mart}
\end{align*}
where $\Dvk_h^k(s_h^k) \coloneqq V_h^{k}(s_h^k) - V_h^{\pii^k}(s_h^k)$. We express $\Dvk_h^k(s_h^k)$ recursively, introducing terms that can be bounded either via concentration inequalities for $\hat{\pp}_k$ around $\pp^*$ or by the bonus functions.
\begin{align*}
\Dvk_h^k(s_h^k) \mathbbm{1}\{\mathcal{G}^k\}
&\leq \mathbbm{1}\{\mathcal{G}^k\} b_h^k(n_h^k(s_h^k,a_h^k))  + \mathbbm{1}\{\mathcal{G}^k\} \sum_{s'} \left( \hat{P}^k_h(s'\mid s_h^k,a_h^k) - P^*_h(s'\mid s_h^k,a_h^k)\right) V_{h+1}^{k}(s') \\
& \qquad +\mathbbm{1}\{\mathcal{G}^k\} \sum_{s'} P^*_h(s'\mid s_h^k,a_h^k) \left( V_{h+1}^{k}(s')  - V_{h+1}^{\pii^k}(s')\right) \\
& \stackrel{(a)}{\leq} \mathbbm{1}\{\mathcal{G}^k\} b_h^k(n_h^k(s_h^k,a_h^k))  + \frac{H-h-1}{2} \mathbbm{1}\{\mathcal{G}^k\} \|P_h^*(\cdot \mid s_h^k, a_h^k) - \hat{P}_h^k(\cdot \mid s_h^k, a_h^k)\|_1  \\
&\qquad+ \mathbbm{1}\{\mathcal{G}^k\}\sum_{s'} P^*_h(s'\mid s_h^k,a_h^k) \Dvk_{h+1}^k(s')\\
& \leq 2 \mathbbm{1}\{\mathcal{G}^k\} b^k_h(n_h^k(s_h^k, a_h^k))  + \mathbbm{1}\{\mathcal{G}^k\} \sum_{s'} P^*_h(s'\mid s_h^k,a_h^k) \Dvk_{h+1}^k(s') \\
& = 2 \mathbbm{1}\{\mathcal{G}^k\} b^k_h(n_h^k(s_h^k, a_h^k))  + \mathbb{E}\left[ \mathbbm{1}\{\mathcal{G}^k\}\Dvk_{h+1}^{k}(s_{h+1}^k) \mid \mathcal{F}_h^k \right]\\
& = 2 \mathbbm{1}\{\mathcal{G}^k\} b^k_h(n_h^k(s_h^k, a_h^k))  + \Dvk_{h+1}^{k}(s_{h+1}^k)\mathbbm{1}\{\mathcal{G}^k\} - \Dvkm_{h+1}^{k}(s_{h+1}^k)\mathbbm{1}\{\mathcal{G}^k\}\numberthis\label{eq:Dvk-bound}.
\end{align*}
where $(a)$ follows from H\"{o}lder's inequality, and noting that $\sum_{s} ( P_h^*(s \mid s_h^k, a_h^k) - \hat{P}_h^k(s \mid s_h^k, a_h^k) ) = 0$, and the sequence $\{\Dvkm^k_{h}(s_{h}^k)\}_{h=1,k=0}^{H-1,K-1}$ defined as 
\begin{align*}
    \Dvkm^k_{h}(s_{h}^k) = \Dvk^k_{h}(s_{h}^k) - \E\left[\Dvk^k_{h}(s_{h}^k)\big| \mathcal{F}_{h-1}^{k} \right],
\end{align*}
is a martingale difference sequence with respect to the filtration $\{\mathcal{F}_{h-1}^{k}\}_{h=1,k=0}^{H-1,K-1}$. For the sake of notational simplicity, we define $\Dvkm^k_{H}\equiv 0$ for all $k\in[K]$. Summing \eqref{eq:Dvk-bound} over $\{h,h+1,\cdots,H-1\}$, the terms telescope, yielding
\begin{align}
    \Dvk_h^k(s_h^k)\mathbbm{1}\{\mathcal{G}^k\} \leq 2 \mathbbm{1}\{\mathcal{G}^k\} \sum_{l=h}^{H-1}  b^k_l(n_l^k(s_l^k, a_l^k))  -  \mathbbm{1}\{\mathcal{G}^k\} \sum_{l=h}^{H-1} \Dvkm_{l+1}^{k}(s_{l+1}^k).\label{eq:Dvk-finalBound}
\end{align}
Notice that the action taken at horizon $\upsilon^k$ is suboptimal. Hence, we have
\begin{align*}
    \Dvk_{\upsilon^k}^k(s_{\upsilon^k}^k)\mathbbm{1}\{\mathcal{G}^k\} &= \left(V_{\upsilon^k}^{k}(s_{\upsilon^k}^k) -V_{\upsilon^k}^{\pii^k}(s_{\upsilon^k}^k) \right) \mathbbm{1}\{\mathcal{G}^k\}\\
    &\geq \left( V_{\upsilon^k}^{*}(s_{\upsilon^k}^k) -V_{\upsilon^k}^{\pii^k}(s_{\upsilon^k}^k)\right) \mathbbm{1}\{\mathcal{G}^k\} \tag{by \eqref{eq:good_event}}\\
    &\geq \left( V_{\upsilon^k}^{*}(s_{\upsilon^k}^k) -Q_{\upsilon^k}^{*}(s_{\upsilon^k}^k,a_{\upsilon^k}^k) \right) \mathbbm{1}\{\mathcal{G}^k\}\tag{by $V_{\upsilon^k}^{\pii^k}(s_{\upsilon^k}^k) = Q_{\upsilon^k}^{\pii^k}(s_{\upsilon^k}^k,a_{\upsilon^k}^k)\leq Q_{\upsilon^k}^{*}(s_{\upsilon^k}^k,a_{\upsilon^k}^k)$  }\\
    &=\gap_{\upsilon^k}(s_{\upsilon^k}^k,a_{\upsilon^k}^k) \\
    &\geq \gapstar,
\end{align*}
which implies that $\Dvk_{\upsilon^k}^k(s_{\upsilon^k}^k) \mathbbm{1}\{\mathcal{G}^k\}= \clip\left(\Dvk_{\upsilon^k}^k(s_{\upsilon^k}^k)\right)\mathbbm{1}\{\mathcal{G}^k\}$. Substituting this into $\eqref{eq:R_K_1_mart}$ yields
\begin{align*}
    R_K^{(1)} \mathbbm{1}\{\bar{\mathcal{G}}^K\}
    &\leq \sum_{k=\lceil K^\gamma\rceil}^{K-1}\left[ \clip\left(\Dvk_{\upsilon^k}^k(s_{\upsilon^k}^k) \right) \mathbbm{1}\{\mathcal{G}^k\} - \sum_{h=0}^{{\upsilon^k}-1}\Dvsm^k_{h+1}(s_{h+1}^k)\right]\mathbbm{1}\{\bar{\mathcal{G}}^K\}\\
    &\leq \sum_{k=\lceil K^\gamma\rceil}^{K-1}\left[ \clip\left(2 \sum_{l=\upsilon^k}^{H-1}  b^k_l(n_l^k(s_l^k, a_l^k))  -  \sum_{l=\upsilon^k}^{H-1} \Dvkm_{l+1}^{k}(s_{l+1}^k) \right) - \sum_{h=0}^{{\upsilon^k}-1}\Dvsm^k_{h+1}(s_{h+1}^k)\right]\mathbbm{1}\{\bar{\mathcal{G}}^K\} \tag{by \eqref{eq:Dvk-finalBound}}\\
     &\leq \sum_{k=\lceil K^\gamma\rceil}^{K-1}\left[ \clip\left(4 \sum_{l=\upsilon^k}^{H-1}  b^k_l(n_l^k(s_l^k, a_l^k))\right)  - 2 \sum_{l=\upsilon^k}^{H-1} \Dvkm_{l+1}^{k}(s_{l+1}^k) - \sum_{h=0}^{{\upsilon^k}-1}\Dvsm^k_{h+1}(s_{h+1}^k)\right]\mathbbm{1}\{\bar{\mathcal{G}}^K\} \tag{by Lemma \ref{lem:clip_ineq}}\\
     &= \sum_{k=\lceil K^\gamma\rceil}^{K-1}\left[ \clip\left(4 \sum_{l=\upsilon^k}^{H-1}  b^k_l(n_l^k(s_l^k, a_l^k))\right)  - \sum_{l=0}^{H-1} \Dvm_{l+1}^{k}(s_{l+1}^k) \right]\mathbbm{1}\{\bar{\mathcal{G}}^K\} \numberthis \label{eq:R_K_1_final_bound}
\end{align*} 
where the sequence $\{\Dvm^k_{h}(s_{h}^k)\}_{h=1,k=0}^{H-1,K-1}$ defined as 
\begin{align*}
    \Dvm_{h}^k (s_{h}^k) \coloneqq \Dvsm_{h}^k(s_{h}^k) \mathbbm{1}\{h \leq \upsilon^k\} + 2\Dvkm_{h}^k(s_{h}^k) \mathbbm{1}\{h > \upsilon^k\}.
\end{align*}
is a martingale difference sequence with respect to the filtration $\{\mathcal{F}_{h-1}^{k}\}_{h=1,k=0}^{H-1,K-1}$. Crucially, because $\upsilon^k$ is the step of the first suboptimal action, the event $\{h \leq \upsilon^k\}$ depends only on whether actions $a_0^k, \dots, a_{h-1}^k$ are optimal, meaning both $\mathbbm{1}\{h \leq \upsilon^k\}$ and $\mathbbm{1}\{h > \upsilon^k\}$ are $\mathcal{F}_{h-1}^k$-measurable. Therefore, $\Dvm_{h}^k$ is a valid martingale difference sequence with respect to $\mathcal{F}_{h-1}^k$. Also, we have $\mathbbm{1}\{\mathcal{G}^k\}  | \Dvm_{h}^k (s_{h}^k)| \leq 2(H-h)$.
Next, we bound the first term in \eqref{eq:R_K_1_final_bound}. We have
\begin{align*}
     \sum_{k=\lceil K^\gamma\rceil}^{K-1} \clip\left(4 \sum_{l=\upsilon^k}^{H-1}  b^k_l(n_l^k(s_l^k, a_l^k))\right) &\leq \sum_{k=0}^{K-1} \clip\left(4 \sum_{l=0}^{H-1}  b^k_l(n_l^k(s_l^k, a_l^k))\right) \\
     &= 4\sum_{k=0}^{K-1}\sum_{n=1}^N  w^k(n)\sum_{l=0}^{H-1}   b^k_l(n_l^k(s_l^k, a_l^k))\\
     &= 4\sum_{n=1}^N\sum_{k=0}^{K-1}  w^k(n)\sum_{l=0}^{H-1}   b^k_l(n_l^k(s_l^k, a_l^k))\numberthis\label{eq:16_further_bound}
\end{align*}
where $w^k(n) \coloneqq \mathbbm{1}\{4 \sum_{l=0}^{H-1}  b^k_l(n_l^k(s_l^k, a_l^k)) \in [2^{n-1}\gapstar,2^n\gapstar)\}$, $N = \lceil\log_2\left(4 H b_{\max}/\gapstar\right)\rceil$, and $b_{\max} = H\sqrt{0.5 S \ln 2 + \mu K^\alpha}$. For a fixed $n\in[N+1]$, we have
\begin{align*}
    \sum_{k=0}^{K-1} w^k(n) \sum_{l=0}^{H-1} b^k_l(n_l^k(s_l^k, a_l^k)) &= \sum_{l=0}^{H-1}\sum_{k=0}^{K-1} w^k(n)  b^k_l(n_l^k(s_l^k, a_l^k))\\
    &= \sum_{l=0}^{H-1} \sum_{(s,a)}\,  \sum_{\substack{k=0\\(s_l^k, a_l^k) = (s,a)} }^K w^k(n) b^k_l(n_l^k(s, a))\\
    &\leq \sum_{l=0}^{H-1} \sum_{(s,a)} \, \sum_{i=1}^{n_l^K(s, a)}  w^{\tau(s,a,i)}(n) \frac{b_{\max}}{\sqrt{i}}\\
    &\leq \sum_{l=0}^{H-1}\sum_{(s,a)} \, \sum_{i=1}^{C_l{(s,a;n)}} \frac{b_{\max}}{\sqrt{i}}\\
    &\leq 2 b_{\max}\sum_{l=0}^{H-1} \sum_{(s,a)} \sqrt{C_l{(s,a;n)}}\\
    &\leq 2 b_{\max} \sqrt{H S AC(n)} \numberthis\label{eq:C_n_sqrt}
\end{align*}
where the last inequality follows by Cauchy-Schwartz and $C(n) = \sum_{(s,a)}\sum_{l=0}^{H-1} C_l{(s,a;n)}$, $C_l{(s,a;n)} = \sum_{i=1}^{n_l^K(s, a)} w^{\tau(s,a,i)}(n)$, and $\tau(s,a,i)$ denotes the episode index in which action $(s,a)$ at horizon $l$ is taken for the $i$-th time. On the other hand, we have the following lower bound:
\begin{align}
    \sum_{k=0}^{K-1} w^k(n)\sum_{l=0}^{H-1}   4 b^k_l(n_l^k(s_l^k, a_l^k)) \geq 2^{n-1}\gapstar C(n).\label{eq:C_n}
\end{align}
Hence, combining \eqref{eq:C_n_sqrt} and \eqref{eq:C_n}, we have 
\begin{align*}
    C(n) \leq  \frac{b^2_{max} H S A}{4^{n-4}(\gapstar)^2}
\end{align*}
Hence, \eqref{eq:16_further_bound} is bounded as follows:
\begin{align*}
     \sum_{k=\lceil K^\gamma\rceil}^{K-1} \clip\left(4 \sum_{l=\upsilon^k}^{H-1}  b^k_l(n_l^k(s_l^k, a_l^k))\right) &\leq \sum_{n=1}^N \frac{8 b^2_{\max} H S A}{ 2^{n-4} \gap^*} \leq \frac{128 H^3 S A (0.5 S \ln 2 + \mu K^\alpha)}{\gapstar}  \eqqcolon m_K^{\mathrm{mb}} \numberthis\label{eq:m_k^mb-ineq}
\end{align*}
Combining the bounds, we have
\begin{align*}
    &\PP\left( \left\{R_K^{(1)} \mathbbm{1}\{\bar{\mathcal{G}}^K\} \geq x-\lceil K^\gamma\rceil H\right\}\right) \\
    &\qquad\qquad\leq\PP\left( \mathbbm{1}\{\bar{\mathcal{G}}^K\}\left(m_K^{\mathrm{mb}} - \sum_{k=\lceil K^\gamma\rceil}^{K-1}\sum_{l=0}^{H-1} \Dvm_{h+1}^k \right) \geq  x-\lceil K^\gamma\rceil H \right)\\
    &\qquad\qquad\leq\PP\left(  - \sum_{k=\lceil K^\gamma\rceil}^{K-1}\sum_{l=0}^{H-1} \Dvm_{h+1}^k \mathbbm{1}\{\mathcal{G}^k\} \geq  x-\lceil K^\gamma\rceil H - m_K^{\mathrm{mb}}\right)\tag{since $x > \lceil K^\gamma\rceil H$}\\
    &\qquad\qquad\leq \exp \left( - \frac{(x-\lceil K^\gamma\rceil H - m_K^{\mathrm{mb}})^2}{4H(H+1)(2H+1)K/3}\right).
\end{align*}
where the last inequality follows by applying Azuma–Hoeffding.\hfill$\square$

\subsection{Proof of Proposition \ref{prop:expectGap-UCBVI}}
    We have
    \begin{align}
        \E[R_K] = & \E[R_K^{(0)}] + \E[R_K^{(1)}] \nonumber\\
        \leq &H\lceil K^\gamma\rceil + \E\left[R_K^{(1)}  \left(\mathbbm{1}\left\{\bar{\mathcal{G}}^K\right\} + \mathbbm{1}\left\{\left(\bar{\mathcal{G}}^K\right)^\mathrm{c}\right\}\right)\right] \nonumber\\
        = &H\lceil K^\gamma\rceil + \E \left[R_K^{(1)} \mathbbm{1}\left\{\bar{\mathcal{G}}^K\right\}\right] + \E\left[R_K^{(1)} \mathbbm{1}\left\{\left(\bar{\mathcal{G}}^K\right)^c\right\}\right] \nonumber\\
        \leq &H\lceil K^\gamma\rceil + \E \left[R_K^{(1)} \mathbbm{1}\left\{\bar{\mathcal{G}}^K\right\}\right] + H(K-\lceil K^\gamma\rceil)\PP\left(\left(\bar{\mathcal{G}}^K\right)^\mathrm{c}\right).\label{eq:upperBoundExpTemp-UCBVI}
    \end{align}
    Note that by \eqref{eq:R_K_1_final_bound} and \eqref{eq:m_k^mb-ineq}, we have
    \begin{align*}  
    R_K^{(1)} \mathbbm{1}\{\bar{\mathcal{G}}^K\} \leq   \mathbbm{1}\{\bar{\mathcal{G}}^K\}\left(m_K^{\mathrm{mb}} - \sum_{k=\lceil K^\gamma\rceil}^{K-1}\sum_{h=0}^{H-1} \Dvm_{h+1}^k \right)
    \end{align*}
    Hence, we obtain
    \begin{align*}
        \E \left[R_K^{(1)} \mathbbm{1}\left\{\bar{\mathcal{G}}^K\right\}\right] 
        &\leq m_K^{\mathrm{mb}} + \E\left[\left( - \sum_{k=\lceil K^\gamma\rceil}^{K-1}\sum_{h=0}^{H-1} \Dvm_{h+1}^k\right) \mathbbm{1}\left\{\bar{\mathcal{G}}^K\right\}\right]\\
        &=m_K^{\mathrm{mb}} + \E\left[ - \sum_{k=\lceil K^\gamma\rceil}^{K-1}\sum_{h=0}^{H-1} \Dvm_{h+1}^k\right] + \E\left[\left( \sum_{k=\lceil K^\gamma\rceil}^{K-1}\sum_{h=0}^{H-1} \Dvm_{h+1}^k\right) \mathbbm{1}\left\{\left(\bar{\mathcal{G}}^K\right)^{\mathrm{c}}\right\}\right]\\
        &\leq m_K^{\mathrm{mb}} + KH(H+1)\delta_K^{\mathrm{mb}}(\gamma)\numberthis\label{eq:upperBoundR1R2Temp-UCBVI}.
    \end{align*}

    Combining \eqref{eq:upperBoundExpTemp-UCBVI} and \eqref{eq:upperBoundR1R2Temp-UCBVI}, we get
    \begin{align*}
        \E[R_K]\leq  H\lceil K^\gamma\rceil + m_K^{\mathrm{mb}} + KH(H+1)\delta_K^{\mathrm{mb}}(\gamma).
\end{align*}\hfill$\square$

\subsection{Proof of Theorem~\ref{thm:tailGap-Qlearning}}
Fix state-action pair $(s,a)$, horizon $h$ and episode number $k$. Consider the sequence of episodes $\tau_h(s,a,1),\tau_h(s,a,2),\dots$ in which
$(s_h^{\tau_h(s,a,i)},a_h^{\tau_h(s,a,i)})=(s,a)$ for all $i\geq 1$.
Let $t=n_h^k(s,a)$ be the visit count to $(s,a,h)$ up to the end of episode $k$, and define
\begin{align*}
Y_i \;\coloneqq\; r_h(s,a) + V_{h+1}^{\tau_h(s,a,i)}\!\Big(s_{h+1}^{\tau_h(s,a,i)}\Big) + b_h(i),
\qquad 1\le i\le t.
\end{align*}

By unrolling the stochastic approximation recursion (using \eqref{eq:alpha-weights}),
the value stored in the table after $t$ visits satisfies
\begin{align}
Q_h^k(s,a)
\;=\;
\lr^{0\rightarrow t}\cdot H + \sum_{i=1}^t \lr^{i\rightarrow t} \, Y_i.
\label{eq:unroll-Q}
\end{align}

Next, define 
\begin{align*}
\xi_i
\;\coloneqq\;
V_{h+1}^\star\!\Big(s_{h+1}^{\tau_h(s,a,i)}\Big)
-\sum_{s'}V_{h+1}^\star(s') P_h^*(s'\mid s,a),
\qquad 1\le i\le t,
\end{align*}
which is a bounded martingale difference noise with respect to the filtration $\{\mathcal{F}^{\tau_h(s,a,t)}_h\}_{h=0,t=1}^{H-1,K}$, i.e., we have $\E[\xi_i\mid \mathcal{F}^{\tau_h(s,a,i)}_h]=0$ and
$|\xi_i|\le H$ since $V_{h+1}^\star\in[0,H]$.

For a fixed $t \geq 1$, define the per-visitation good event $\mathcal{E}^t_h(s,a)$ as 
\[
\mathcal{E}^{t}_h(s,a) \;\coloneqq\; \left\{
\Big|\sum_{i=1}^t \lr^{i\rightarrow t} \xi_i\Big|
\le
\sqrt{2H^2\left(\sum_{i=1}^t(\lr^{i\rightarrow t})^2\right)\,\iota}
\right\},
\]
where $\iota=\ln(2SAHK)+\mu K^\alpha$. Since the sequence $\{\xi_i\}$ is a martingale difference sequence with $|\xi_i|\le H$, the weighted sum $\{ \sum_{i=1}^j \lr^{i\rightarrow t} \xi_i\}_{j=1}^t$ is again a martingale with respect to the same filtration, and with increments bounded by $H\lr^{j\rightarrow t}$. Azuma--Hoeffding yields
\[
\PP\!\left(\left|\sum_{i=1}^t \lr^{i\rightarrow t} \xi_i\right|\ge \eps\right)
\le 2\exp\!\left(-\frac{\eps^2}{2H^2\sum_{i=1}^t(\lr^{i\rightarrow t})^2}\right).
\]
Setting $\eps=\sqrt{2H^2(\sum(\lr^{i\rightarrow t})^2)\iota}$ gives $\PP((\mathcal{E}^t_h(s,a))^c)\le 2 e^{-\iota}$. By a union bound over $SAH$ tables and at most $K$ visit counts, we have 
\begin{align*}
    \PP(\bar{\mathcal{E}}^c)\;\le 2SAKH e^{-\iota} = e^{-\mu K^\alpha} = \delta_K^{\mathrm{mf}}
\end{align*}
where $\bar{\mathcal{E}}^K \coloneqq \bigcap_{t=1}^{K} \bigcap_{s,a}\bigcap_h \mathcal{E}^t_h(s,a)$ is the good event. By Lemma~\ref{lem:weights}, $\sum_{i=1}^t(\lr^{i\rightarrow t})^2\le 2H/t$, which implies that under the good event $\bar{\mathcal{E}}^K$, we have 
\begin{align}    \label{eq:eta_sum_bound}
\left|\sum_{i=1}^t \lr^{i\rightarrow t} \xi_i\right|
\;\le\;
\sqrt{\frac{4H^3\iota}{t}}.
\end{align}

\begin{lemma}
\label{lem:optimism}
On $\bar{\mathcal{E}}^K$, for all $h\in[H]$, and all $(s,a)\in\mathcal{S}\times\mathcal{A}$,
\[
Q_h^k(s,a)\;\ge\;Q_h^\star(s,a),
\qquad\text{and hence}\qquad
V_h^k(s)\;\ge\;V_h^\star(s),
\]
where $Q_h^k$ and $V_h^k$ are as in Algorithm \ref{alg:qlearning-tail}. Furthermore,
\begin{align*}
    Q_h^k(s,a) - Q_h^\star(s,a) \leq \lr^{0\rightarrow t}\cdot H + \sum_{i=1}^t \lr^{i\rightarrow t} \, \left(V_{h+1}^{\tau_h(s,a,i)}(s_{h+1}^{\tau_h(s,a,i)})-V_{h+1}^\star(s_{h+1}^{\tau_h(s,a,i)})\right) + 2 b_h(t).
\end{align*}
\end{lemma}

\begin{proof}[Proof of Lemma \ref{lem:optimism}]
The proof follows by backward induction on the horizon. By definition, we have $V_{H}^k\equiv 0$, and hence $V_H^k(s)=V^*_H(s)$ for all states $s\in\mathcal{S}$ and all episodes $k\geq 0$. 

Suppose that for some $h\leq H-1$, we have $V_{h+1}^k(s)\geq V^*_{h+1}(s)$ for all states $s\in\mathcal{S}$ and all episodes $k\geq 0$. Fix state-action pair $(s,a)$ and episode number $k$, and consider the unrolled form \eqref{eq:unroll-Q} at episode $k$. We have
\begin{align}
Q_h^k(s,a)-Q_h^*(s,a)
\;=&\;
\lr^{0\rightarrow t}\cdot (H-Q_h^*(s,a)) + \sum_{i=1}^t \lr^{i\rightarrow t} \, (Y_i-Q_h^*(s,a))\nonumber\\
\geq &\; \sum_{i=1}^t \lr^{i\rightarrow t} \, \left(Y_i-r_h(s,a)-\sum_{s'}V_{h+1}^\star(s') P_h^*(s'\mid s,a)\right),\label{eq:Q_dif_eta}
\end{align}
where the last inequality is due to the Bellman equation and the fact that $Q^*_h(s,a)\leq H$. Notice that
\[
Y_i - \Big(r_h(s,a)+\sum_{s'}V_{h+1}^\star(s') P_h^*(s'\mid s,a)\Big)
=
V_{h+1}^{\tau_h(s,a,i)}(s_{h+1}^{\tau_h(s,a,i)})-V_{h+1}^\star(s_{h+1}^{\tau_h(s,a,i)})
+\xi_i + b_h(i).
\]
Substituting the above in Eq. \eqref{eq:Q_dif_eta}, we get 
\begin{align}
Q_h^k(s,a)-Q_h^*(s,a)\geq &\; \sum_{i=1}^t \lr^{i\rightarrow t} \, \left(V_{h+1}^{\tau_h(s,a,i)}(s_{h+1}^{\tau_h(s,a,i)})-V_{h+1}^\star(s_{h+1}^{\tau_h(s,a,i)})
+\xi_i + b_h(i)\right)\nonumber\\
\geq & \sum_{i=1}^t \lr^{i\rightarrow t} \, \left(
\xi_i + b_h(i)\right)\tag{by induction hypothesis}\\
\geq & -\sqrt{\frac{4H^3(\ln(2SAHK)+\mu K^\alpha)}{t}} +  \left(\sum_{i=1}^t \lr^{i\rightarrow t} b_h(i)\right) \tag{by \eqref{eq:eta_sum_bound}}\\
\geq & 0 \tag{by Lemma \ref{lem:weights}}
\end{align}
This yields $Q_h^k(s,a)\ge Q_h^\star(s,a)$ for all $k\in[K]$, which implies $V_h^k(s)\ge V_h^\star(s)$ for all $k\in[K]$. To get the upper bound, note that
\begin{align*}
    Q_h^k(s,a) - Q_h^\star(s,a) &\leq \lr^{0\rightarrow t}\cdot H  + \sum_{i=1}^t \lr^{i\rightarrow t} \, \left(Y_i-r_h(s,a)-\sum_{s'}V_{h+1}^\star(s') P_h^*(s'\mid s,a)\right)\\
    &= \lr^{0\rightarrow t}\cdot H  + \sum_{i=1}^t \lr^{i\rightarrow t} \, \left(V_{h+1}^{\tau_h(s,a,i)}(s_{h+1}^{\tau_h(s,a,i)})-V_{h+1}^\star(s_{h+1}^{\tau_h(s,a,i)}) + \xi_i + b_h(i)\right)\\
    &\leq \lr^{0\rightarrow t}\cdot H + \sum_{i=1}^t \lr^{i\rightarrow t} \, \left(V_{h+1}^{\tau_h(s,a,i)}(s_{h+1}^{\tau_h(s,a,i)})-V_{h+1}^\star(s_{h+1}^{\tau_h(s,a,i)})\right) + 2 b_h(t).
\end{align*}
where the last inequality follows by the fact that
\begin{align*}
    \sum_{i=1}^t \lr^{i\rightarrow t} \, \left( \xi_i + b_h(i)\right) &\leq \sqrt{\frac{4H^3(\ln(2SAHK)+\mu K^\alpha)}{t}} + \sum_{i=1}^t \lr^{i\rightarrow t}b_h(i)\tag{by \eqref{eq:eta_sum_bound}}\\
    &\leq 3\sqrt{\frac{4H^3(\ln(2SAHK)+\mu K^\alpha)}{t}} \tag{by Lemma~\ref{lem:weights}}
\end{align*}
\end{proof}

Lemma \ref{lem:optimism} provides the standard optimism property for $Q$-function estimation. Next, we aim to use Lemma \ref{lem:optimism} to upper bound the total regret $R_K$.

Similar to the analysis presented in the proof of Theorem~\ref{thm:tailGap-UCBVI}, by \eqref{eq:Dvs-equality} regret at episode $k\in[K]$ is bounded as follows:
\begin{align*}
    V_0^{*}(s_0^k) - V_0^{\pii^k}(s_0^k) = \Dvs^k_0(s_0^k) =\sum_{h=0}^{H-1}\clip\left(\gap_h(s_h^k,a_h^k)\right) - \sum_{h=0}^{H-1}\Dvsm^k_{h+1}(s_{h+1}^k)
\end{align*}
where the sequence $\{\Dvsm^k_{h}(s_{h}^k)\}_{h=1,k=0}^{H-1,K-1}$ defined as 
\begin{align*}
    \Dvsm^k_{h}(s_{h}^k) = \Dvs^k_{h}(s_{h}^k) - \E\left[\Dvs^k_{h}(s_{h}^k)\big| \mathcal{F}_{h-1}^{k} \right]
\end{align*}
is a martingale difference sequence with respect to the filtration $\{\mathcal{F}_{h-1}^{k}\}_{h=1,k=0}^{H-1,K-1}$. 

Recall that $R_K$ denote the cumulative regret incurred up to episode $K-1$. Then, on the good event $\bar{\mathcal{E}}^K$, we have:
\begin{align}
    R_K\mathbbm{1}\{\bar{\mathcal{E}}^K\} &= \sum_{k=0}^{K-1}\left[\sum_{h=0}^{H-1}\clip\left(\gap_h(s_h^k,a_h^k)\right) - \sum_{h=0}^{H-1}\Dvsm^k_{h+1}(s_{h+1}^k)\right]\mathbbm{1}\{\bar{\mathcal{E}}^K\}\nonumber\\
    &=\sum_{k=0}^{K-1}\Bigg[\sum_{h=0}^{H-1}\clip\left(V_h^{*}(s_h^k) \pm V_h^{k}(s_h^k) - Q_h^{*}(s_h^k,a_h^k)\right) - \sum_{h=0}^{H-1}\Dvsm^k_{h+1}(s_{h+1}^k)\Bigg]\mathbbm{1}\{\bar{\mathcal{E}}^K\}\nonumber\\
    &\leq \sum_{k=0}^{K-1}\left[\sum_{h=0}^{H-1}\clip\left(Q_h^{k}(s_h^k,a_h^k) - Q_h^{*}(s_h^k,a_h^k)\right) - \sum_{h=0}^{H-1}\Dvsm^k_{h+1}(s_{h+1}^k)\right]\mathbbm{1}\{\bar{\mathcal{E}}^K\}\label{eq:18_before}
\end{align}
where the last inequality follows by definition of ${\mathcal{E}}^K$ and the fact that $V_h^{k}(s_h^k) = Q_h^{k}(s_h^k,a_h^k)$.

Next, we bound the first term as follows:
\begin{align*}
    \mathbbm{1}\{\bar{\mathcal{E}}^K\}\sum_{k=0}^{K-1}\sum_{h=0}^{H-1}\clip\left(Q_h^{k}(s_h^k,a_h^k) - Q_h^{*}(s_h^k,a_h^k)\right) \leq & \mathbbm{1}\{\bar{\mathcal{E}}^K\}\sum_{k=0}^{K-1}\sum_{h=0}^{H-1} \sum_{n=1}^{N}\left(Q_h^{k}(s_h^k,a_h^k) - Q_h^{*}(s_h^k,a_h^k)\right) w_h^k(n) \\
    =&\mathbbm{1}\{\bar{\mathcal{E}}^K\} \sum_{n=1}^{N}\sum_{h=0}^{H-1}\sum_{k=0}^{K-1} \left(Q_h^{k}(s_h^k,a_h^k) - Q_h^{*}(s_h^k,a_h^k)\right) w_h^k(n)
\end{align*}
where $w_h^k(n) = \mathbbm{1}\{Q_h^{k}(s_h^k,a_h^k) - Q_h^{*}(s_h^k,a_h^k) \in [2^{n-1}\gapstar,2^n\gapstar)\}$ and $N=\ceil{\log_2(H/\gapstar)}$. Applying \cite[Proof of Lemma 4.2]{yang2021q}, we have
\begin{align*}
        \mathbbm{1}\{\bar{\mathcal{E}}^K\} \sum_{k=0}^{K-1} \left(Q_h^{k}(s_h^k,a_h^k) - Q_h^{*}(s_h^k,a_h^k)\right) w_h^k(n)\leq & \mathbbm{1}\{\bar{\mathcal{E}}^K\} \sum_{n=1}^{N}\sum_{h=0}^{H-1} (eSAH^2+ 40\sqrt{eSA C^{(n,h)} H^5\iota}) 
\end{align*}
where $C^{(n,h)} \coloneqq \left|\{k: (Q_h^{k}(s_h^k,a_h^k) - Q_h^{*}(s_h^k,a_h^k)) \in [2^{n-1}\gapstar,2^n\gapstar)\}\right|$ satisfies 
\begin{align*}
 C^{(n,h)} \leq  \frac{6404eSAH^5\iota }{4^{n} (\gapstar)^2}.
\end{align*}
This implies
\begin{align}
    \mathbbm{1}\{\bar{\mathcal{E}}^K\}\sum_{k=0}^{K-1}\sum_{h=0}^{H-1}\clip\left(Q_h^{k}(s_h^k,a_h^k) - Q_h^{*}(s_h^k,a_h^k)\right) \nonumber
    &\leq \mathbbm{1}\{\bar{\mathcal{E}}^K\} \left(eSAH^3 \ceil{\log_2(H/\gapstar)} + \frac{3240 eSAH^6\iota}{\gapstar}\right)\nonumber\\
    &\leq \mathbbm{1}\{\bar{\mathcal{E}}^K\}  \underbrace{\frac{3242 eSAH^6\iota}{\gapstar}}_{\eqqcolon m_K^{\mathrm{mf}}} \label{eq:19_before}
\end{align}
Combining the bounds, we have
\begin{align*}
    \PP\left( \left\{R_K \mathbbm{1}\{\bar{\mathcal{E}}^K\} \geq x \right\}\right) 
    &\leq\PP\left( \mathbbm{1}\{\bar{\mathcal{E}}^K\}\left(m_K^{\mathrm{mf}} - \sum_{k=0}^{K-1} \sum_{h=0}^{H-1}\Dvsm^k_{h+1}(s_{h+1}^k) \right) \geq  x \right)\\
    &\leq\PP\left(-\sum_{k=0}^{K-1}\sum_{h=0}^{H-1} \Dvsm_{h+1}^k(s_{h+1}^k) \geq  x - m_K^{\mathrm{mf}}\right)\\
    &\leq \exp \left( - \frac{(x- m_K^{\mathrm{mf}})^2}{H^3K}\right).
\end{align*}
where the last inequality follows by applying Azuma–Hoeffding. Hence,
\begin{align*}
    \PP(R_K\geq x) &\leq \delta_K^{\mathrm{mf}} + \exp \left( - \frac{(x- m_K^{\mathrm{mf}})^2}{H^3K}\right). 
\end{align*}\hfill$\square$

\subsection{Proof of Proposition \ref{prop:expectGap-UCBQVI}}
    We have
    \begin{align}
        \E[R_K] = & \E\left[R_K  \left(\mathbbm{1}\left\{\bar{\mathcal{E}}^K\right\} + \mathbbm{1}\left\{\left(\bar{\mathcal{E}}^K\right)^\mathrm{c}\right\}\right)\right] \nonumber\\
        = & \E \left[R_K \mathbbm{1}\left\{\bar{\mathcal{E}}^K\right\}\right] + \E\left[R_K \mathbbm{1}\left\{\left(\bar{\mathcal{E}}^K\right)^c\right\}\right] \nonumber\\
        \leq & \E \left[R_K \mathbbm{1}\left\{\bar{\mathcal{E}}^K\right\}\right] + HK\PP\left(\left(\bar{\mathcal{E}}^K\right)^\mathrm{c}\right).\label{eq:upperBoundExpTemp-UCBQVI}
    \end{align}
    Note that by \eqref{eq:18_before} and \eqref{eq:19_before}, we have
    \begin{align*}  
    R_K \mathbbm{1}\{\bar{\mathcal{E}}^K\} \leq   \mathbbm{1}\{\bar{\mathcal{E}}^K\}\left(m_K^{\mathrm{mf}} - \sum_{k=0}^{K-1}\sum_{h=0}^{H-1} \Dvsm^k_{h+1}(s_{h+1}^k) \right)
    \end{align*}
    Hence, we obtain
    \begin{align*}
        \E \left[R_K \mathbbm{1}\left\{\bar{\mathcal{E}}^K\right\}\right] 
        &\leq m_K^{\mathrm{mf}} + \E\left[\left( - \sum_{k=0}^{K-1}\sum_{h=0}^{H-1} \Dvsm^k_{h+1}(s_{h+1}^k)\right) \mathbbm{1}\left\{\bar{\mathcal{E}}^K\right\}\right]\\
        &=m_K^{\mathrm{mf}} + \E\left[ - \sum_{k=0}^{K-1}\sum_{h=0}^{H-1} \Dvsm^k_{h+1}(s_{h+1}^k)\right] + \E\left[\left( \sum_{k=0}^{K-1}\sum_{h=0}^{H-1} \Dvsm^k_{h+1}(s_{h+1}^k)\right) \mathbbm{1}\left\{\left(\bar{\mathcal{E}}^K\right)^{\mathrm{c}}\right\}\right]\\
        &\leq m_K^{\mathrm{mf}} + KH^2\delta_K^{\mathrm{mf}}\numberthis\label{eq:upperBoundR1R2Temp-UCBQVI}.
    \end{align*}

    Combining \eqref{eq:upperBoundExpTemp-UCBQVI} and \eqref{eq:upperBoundR1R2Temp-UCBQVI}, we get
    \begin{align*}
        \E[R_K]\leq   m_K^{\mathrm{mf}} + 2 KH^2\delta_K^{\mathrm{mf}}.
    \end{align*}\hfill$\square$

\newpage
\appendix
\section*{Appendix}
\addcontentsline{toc}{section}{Appendix}

\begin{lemma}\label{lem:sumUpperBound-UCBVI}
The following bound holds:
\begin{align*}
    \sum_{n=1}^K 
        &\exp\left(-  2n  \left(\frac{b_h^{n \vee \lceil K^\gamma\rceil}(n)}{H-h-1} \right)^2\right) \leq 2^{-S}\times
    \begin{cases}
        \displaystyle K \exp\left(-2\mu { K^{\alpha}}\right),
        & \text{\KD}, \\[2ex]
        \displaystyle K \exp\left(-2\mu  K^{\gamma\alpha}\right),
        & \text{\KI}.
    \end{cases}
\end{align*}
\end{lemma}
\begin{proof}[Proof of Lemma \ref{lem:sumUpperBound-UCBVI}]
Recall that the function $b^k_h(\cdot)$ is defined as
\[
b^k_h(n)
=
\begin{cases}
(H-h-1)\sqrt{\dfrac{0.5 S\ln 2+\mu K^\alpha}{n}}, & \text{\KD},\\[2ex]
(H-h-1)\sqrt{\dfrac{0.5 S\ln 2+\mu (k+1)^\alpha}{n}}, & \text{\KI},
\end{cases}
\qquad
b^k_h(0)\coloneqq\infty.
\]
Let $S_K$ denote the summation in Lemma~\ref{lem:sumUpperBound-UCBVI}, i.e.,
\begin{align*}
    S_K \coloneqq \sum_{n=1}^{K} \exp\left(-  2n  \left(\frac{b_h^{n \vee \lceil K^\gamma\rceil}(n)}{H-h-1} \right)^2\right).
\end{align*}
The analysis depends on the choice of $b^k_h(\cdot)$.

\begin{itemize}[leftmargin=0pt]
    \item {\bf \KD:} We have
        \begin{align*}
        S_K &\leq \sum_{n=1}^{K} \exp\left(-S\ln2-2\mu  K^{\alpha}\right) = 2^{-S} K \exp\left({- 2\mu K^{\alpha}}\right) 
        \end{align*}
    \item {\bf \KI:} Define $S^{(1)}_K$ and $S^{(2)}_K$ as follows:
    \begin{align*}
        S^{(1)}_K \coloneqq \sum_{n=1}^{\lceil K^\gamma\rceil} \exp\left(- 2n  \left(\frac{b_h^{\lceil K^\gamma\rceil}(n)}{H-h-1} \right)^2\right),
        \qquad
        S^{(2)}_K \coloneqq \sum_{n=\lceil K^\gamma\rceil + 1}^{K} \exp\left(- 2n \left(\frac{b_h^{n}(n)}{H-h-1} \right)^2\right).
    \end{align*}
    Notice that $S_K = S^{(1)}_K + S^{(2)}_K,$ and that
        \begin{align*}
            b_h^{\lceil K^\gamma\rceil}(n) = (H-h-1)\sqrt{\frac{0.5 S\ln2 + \mu(\lceil K^\gamma\rceil+1)^{\alpha}}{n}},
            \quad
            b_h^n(n) = (H-h-1)\sqrt{\frac{0.5 S\ln2 + \mu( n+1 )^{\alpha}}{n}}.
        \end{align*}
        
        We proceed by bounding $S_K^{(1)}$ and $S_K^{(2)}$. We begin with $S_K^{(1)}$, for which we have
        \begin{align*}
        S_K^{(1)}
        &= \sum_{n=1}^{\lceil K^\gamma\rceil}\exp\left(-2n \left(\frac{0.5S\ln2 + \mu(\lceil K^\gamma\rceil+1)^{\alpha}}{n}\right)\right) \leq 2^{-S}\lceil K^\gamma\rceil \exp\left(-2\mu K^{\gamma\alpha}\right) 
        \end{align*}
    \end{itemize}

We next bound $S_K^{(2)}$:
\begin{align*}
S^{(2)}_K
&=\sum_{n=\lceil K^\gamma\rceil + 1}^{K} \exp\left(- 2n \left(\frac{b_h^{n}(n)}{H-h-1} \right)^2\right)\\
&\le\sum_{n= \lceil K^\gamma\rceil+1}^{K}2^{-S}\exp\left(-2\mu  n^{\alpha}\right) \\
&\leq 2^{-S} (K - \lceil K^\gamma\rceil)\exp\left(-2\mu  K^{\gamma\alpha}\right).
\end{align*}
Putting everything together, we obtain the desired result.
\end{proof}

\begin{lemma}
\label{lem:weights}
For the weights in \eqref{eq:alpha-weights} with $\lr^{(t)}=(H+1)/(H+t+1)$:
\begin{flalign*}
\textnormal{(a)}&\qquad\max_{i\leq  t}\lr^{i\rightarrow t} \leq \frac{H+1}{H+t+1},&&\\
\textnormal{(b)}&\qquad\sum_{i=0}^t (\lr^{i\rightarrow t})^2 \leq \frac{H+1}{H+t+1}&&\\
\textnormal{(c)}& \qquad\sum_{i=1}^t \frac{\lr^{i\rightarrow t}}{\sqrt{i}} \in  \left[\frac{1}{2\sqrt{t}},\frac{1}{\sqrt{t}}\right]&&
\end{flalign*}
\end{lemma}

\begin{proof}[Proof of Lemma \ref{lem:weights}]
(a) We have
\begin{align*}    
\lr^{i\rightarrow t}=&\frac{H+1}{H+i+1}\prod_{j=i+1}^t\left(1-\frac{H+1}{H+j+1}\right)\\
=&\frac{H+1}{H+i+1}\prod_{j=i+1}^t\frac{j}{H+j+1}\\
=&\frac{H+1}{H+t+1}\cdot\prod_{j=i+1}^t\frac{j}{H+j}\\
\leq & \frac{H+1}{H+t+1}
\end{align*}
(b) For the squared-sum bound, we have 
\begin{align*}
\sum_{i=0}^t (\lr^{i\rightarrow t})^2 &\leq \frac{H+1}{H+t+1}\sum_{i=0}^t \lr^{i\rightarrow t} \\
&= \frac{H+1}{H+t+1}.\tag{since $\sum_{i=0}^t \lr^{i\rightarrow t}=1$}
\end{align*}
(c) For $t=1$, the result is clear. Now, suppose the result holds for $t\geq 1$. We have $\lr^{i\rightarrow t+1} = \lr^{i\rightarrow t}(1-\lr^{(t+1)})$. Hence,
\begin{align*}
    \sum_{i=1}^{t+1} \frac{\lr^{i\rightarrow t+1}}{\sqrt{i}} &= \frac{\lr^{(t+1)}}{\sqrt{t+1}} + (1-\lr^{(t+1)})\sum_{i=1}^{t} \frac{\lr^{i\rightarrow t}}{\sqrt{i}} \\
    & \in \frac{\lr^{(t+1)}}{\sqrt{t+1}} + (1-\lr^{(t+1)}) \left[\frac{1}{2\sqrt{t}},\frac{1}{\sqrt{t}}\right]\\
    & \in \left[\frac{1}{2\sqrt{t+1}},\frac{1}{\sqrt{t+1}}\right],
\end{align*}
which completes the result.
\end{proof}

\begin{lemma}\label{lem:clip_ineq}
Let $y \ge 0$ and $z \in \mathbb{R}$. If $x \le y - z$ and $x\geq \gapstar$, then
    \[ x = \clip(x) \le \clip(2y) - 2z. \]
\end{lemma}
\begin{proof}
    If $2y \ge \gapstar$, then the inequality is trivial. Suppose that $2y < \gapstar$, then $\clip(2y) = 0$. Because $x \ge \gapstar$, we have $2y < x$, implying $-2y > -x$. Therefore, $-2z \ge 2x - 2y > 2x - x = x$. Thus, $\clip(2y) - 2z = -2z > x$, completing the proof.
\end{proof}

\bibliographystyle{plainnat} 
\bibliography{reference} 

\begin{thebibliography}{25}
\providecommand{\natexlab}[1]{#1}
\providecommand{\url}[1]{\texttt{#1}}
\expandafter\ifx\csname urlstyle\endcsname\relax
  \providecommand{\doi}[1]{doi: #1}\else
  \providecommand{\doi}{doi: \begingroup \urlstyle{rm}\Url}\fi

\bibitem[Azar et~al.(2017)Azar, Osband, and Munos]{azar2017minimax}
Mohammad~Gheshlaghi Azar, Ian Osband, and R{\'e}mi Munos.
\newblock Minimax regret bounds for reinforcement learning.
\newblock In \emph{International conference on machine learning}, pages 263--272. PMLR, 2017.

\bibitem[Bourel et~al.(2020)Bourel, Maillard, and Talebi]{bourel2020tightening}
Hippolyte Bourel, Odalric Maillard, and Mohammad~Sadegh Talebi.
\newblock Tightening exploration in upper confidence reinforcement learning.
\newblock In \emph{International Conference on Machine Learning}, pages 1056--1066. PMLR, 2020.

\bibitem[Bubeck and Cesa-Bianchi(2012)]{bubeck2012regret}
S{\'e}bastien Bubeck and Nicolo Cesa-Bianchi.
\newblock Regret analysis of stochastic and nonstochastic multi-armed bandit problems.
\newblock \emph{Foundations and Trends{\textregistered} in Machine Learning}, 5\penalty0 (1):\penalty0 1--122, 2012.

\bibitem[Chow et~al.(2015)Chow, Tamar, Mannor, and Pavone]{chow2015risk}
Yinlam Chow, Aviv Tamar, Shie Mannor, and Marco Pavone.
\newblock Risk-sensitive and robust decision-making: a cvar optimization approach.
\newblock \emph{Advances in neural information processing systems}, 28, 2015.

\bibitem[Dann et~al.(2021)Dann, Marinov, Mohri, and Zimmert]{dann2021beyond}
Christoph Dann, Teodor~Vanislavov Marinov, Mehryar Mohri, and Julian Zimmert.
\newblock Beyond value-function gaps: Improved instance-dependent regret bounds for episodic reinforcement learning.
\newblock \emph{Advances in Neural Information Processing Systems}, 34:\penalty0 1--12, 2021.

\bibitem[Efroni et~al.(2019)Efroni, Merlis, Ghavamzadeh, and Mannor]{efroni2019tight}
Yonathan Efroni, Nadav Merlis, Mohammad Ghavamzadeh, and Shie Mannor.
\newblock Tight regret bounds for model-based reinforcement learning with greedy policies.
\newblock \emph{Advances in Neural Information Processing Systems}, 32, 2019.

\bibitem[Fan and Glynn(2024)]{fan2024fragility}
Lin Fan and Peter~W Glynn.
\newblock The fragility of optimized bandit algorithms.
\newblock \emph{Operations Research}, 2024.

\bibitem[Fruit et~al.(2020)Fruit, Pirotta, and Lazaric]{fruit2020improved}
Ronan Fruit, Matteo Pirotta, and Alessandro Lazaric.
\newblock Improved analysis of ucrl2 with empirical bernstein inequality.
\newblock \emph{arXiv preprint arXiv:2007.05456}, 2020.

\bibitem[He et~al.(2021)He, Zhou, and Gu]{he2021logarithmic}
Jiafan He, Dongruo Zhou, and Quanquan Gu.
\newblock Logarithmic regret for reinforcement learning with linear function approximation.
\newblock In \emph{International Conference on Machine Learning}, pages 4171--4180. PMLR, 2021.

\bibitem[Inamdar et~al.(2024)Inamdar, Sundarr, Khandelwal, Sahu, and Katal]{inamdar2024comprehensive}
Rohan Inamdar, S~Kavin Sundarr, Deepen Khandelwal, Varun~Dev Sahu, and Nitish Katal.
\newblock A comprehensive review on safe reinforcement learning for autonomous vehicle control in dynamic environments.
\newblock \emph{e-Prime-Advances in Electrical Engineering, Electronics and Energy}, 10:\penalty0 100810, 2024.

\bibitem[Jaksch et~al.(2010)Jaksch, Ortner, and Auer]{jaksch2010near}
Thomas Jaksch, Ronald Ortner, and Peter Auer.
\newblock Near-optimal regret bounds for reinforcement learning.
\newblock \emph{Journal of Machine Learning Research}, 11\penalty0 (51):\penalty0 1563--1600, 2010.

\bibitem[Jin et~al.(2018)Jin, Allen-Zhu, Bubeck, and Jordan]{jin2018q}
Chi Jin, Zeyuan Allen-Zhu, Sebastien Bubeck, and Michael~I Jordan.
\newblock Is q-learning provably efficient?
\newblock \emph{Advances in neural information processing systems}, 31, 2018.

\bibitem[Simchi-Levi et~al.(2022)Simchi-Levi, Zheng, and Zhu]{simchi2022simple}
David Simchi-Levi, Zeyu Zheng, and Feng Zhu.
\newblock A simple and optimal policy design for online learning with safety against heavy-tailed risk.
\newblock \emph{Advances in Neural Information Processing Systems}, 35:\penalty0 33795--33805, 2022.

\bibitem[Simchi-Levi et~al.(2023)Simchi-Levi, Zheng, and Zhu]{simchi2023stochastic}
David Simchi-Levi, Zeyu Zheng, and Feng Zhu.
\newblock Stochastic multi-armed bandits: Optimal trade-off among optimality, consistency, and tail risk.
\newblock \emph{Advances in Neural Information Processing Systems}, 36:\penalty0 35619--35630, 2023.

\bibitem[Simchowitz and Jamieson(2019)]{simchowitz2019non}
Max Simchowitz and Kevin~G Jamieson.
\newblock Non-asymptotic gap-dependent regret bounds for tabular mdps.
\newblock \emph{Advances in Neural Information Processing Systems}, 32, 2019.

\bibitem[Tamar et~al.(2015)Tamar, Glassner, and Mannor]{tamar2015optimizing}
Aviv Tamar, Yonatan Glassner, and Shie Mannor.
\newblock Optimizing the cvar via sampling.
\newblock In \emph{Proceedings of the AAAI Conference on Artificial Intelligence}, volume~29, 2015.

\bibitem[Tian et~al.(2020)Tian, Qian, and Sra]{tian2020towards}
Yi~Tian, Jian Qian, and Suvrit Sra.
\newblock Towards minimax optimal reinforcement learning in factored markov decision processes.
\newblock \emph{Advances in Neural Information Processing Systems}, 33:\penalty0 19896--19907, 2020.

\bibitem[Velegkas et~al.(2022)Velegkas, Yang, and Karbasi]{velegkas2022reinforcement}
Grigoris Velegkas, Zhuoran Yang, and Amin Karbasi.
\newblock Reinforcement learning with logarithmic regret and policy switches.
\newblock \emph{Advances in Neural Information Processing Systems}, 35:\penalty0 36040--36053, 2022.

\bibitem[Weissman et~al.(2003)Weissman, Ordentlich, Seroussi, Verdu, and Weinberger]{weissman2003inequalities}
Tsachy Weissman, Erik Ordentlich, Gadiel Seroussi, Sergio Verdu, and Marcelo~J Weinberger.
\newblock Inequalities for the l1 deviation of the empirical distribution.
\newblock \emph{Hewlett-Packard Labs, Tech. Rep}, page 125, 2003.

\bibitem[Xu et~al.(2021)Xu, Ma, and Du]{xu2021fine}
Haike Xu, Tengyu Ma, and Simon Du.
\newblock Fine-grained gap-dependent bounds for tabular mdps via adaptive multi-step bootstrap.
\newblock In \emph{Conference on Learning Theory}, pages 4438--4472. PMLR, 2021.

\bibitem[Xu et~al.(2023)Xu, Gao, and He]{xu2023regret}
Wenhao Xu, Xuefeng Gao, and Xuedong He.
\newblock Regret bounds for markov decision processes with recursive optimized certainty equivalents.
\newblock In \emph{International Conference on Machine Learning}, pages 38400--38427. PMLR, 2023.

\bibitem[Yang et~al.(2021)Yang, Yang, and Du]{yang2021q}
Kunhe Yang, Lin Yang, and Simon Du.
\newblock Q-learning with logarithmic regret.
\newblock In \emph{International Conference on Artificial Intelligence and Statistics}, pages 1576--1584. PMLR, 2021.

\bibitem[Zanette and Brunskill(2019)]{zanette2019tighter}
Andrea Zanette and Emma Brunskill.
\newblock Tighter problem-dependent regret bounds in reinforcement learning without domain knowledge using value function bounds.
\newblock In Kamalika Chaudhuri and Ruslan Salakhutdinov, editors, \emph{Proceedings of the 36th International Conference on Machine Learning}, volume~97 of \emph{Proceedings of Machine Learning Research}, pages 7304--7312. PMLR, 09--15 Jun 2019.

\bibitem[Zhao et~al.(2025)Zhao, Zhang, Wang, and Li]{zhao2025logarithmic}
Canzhe Zhao, Xiangcheng Zhang, Baoxiang Wang, and Shuai Li.
\newblock Logarithmic regret for linear markov decision processes with adversarial corruptions.
\newblock In \emph{Proceedings of the AAAI Conference on Artificial Intelligence}, volume~39, pages 22759--22767, 2025.

\bibitem[Zheng et~al.(2024)Zheng, Zhang, and Xue]{zheng2025gap}
Zhong Zheng, Haochen Zhang, and Lingzhou Xue.
\newblock Gap-dependent bounds for q-learning using reference-advantage decomposition.
\newblock \emph{arXiv preprint arXiv:2410.07574}, 2024.

\end{thebibliography}
\end{document}